\documentclass[pdflatex,10pt]{article}

\usepackage{Sty/mcr}
\usepackage{Sty/pkg}
\usepackage{Sty/thmE}
\usepackage{Sty/algE}

\title{Selective Inference Approach for \\ Statistically Sound Discriminative Pattern Discovery}

\date{\today}

\author{
Shinya Suzumura \\
Nagoya Institute of Technology \\
\texttt{suzumura.mllab.nit@gmail.com} \\
\and
Kazuya Nakagawa \\
Nagoya Institute of Technology \\
\texttt{nakagawa.k.mllab.nit@gmail.com} \\
\and 
Mahito Sugiyama\\
Osaka University \\
\texttt{mahito@ar.sanken.osaka-u.ac.jp} \\
\and 
Koji Tsuda\\
University of Tokyo \\
\texttt{tsuda@k.u-tokyo.ac.jp} \\
\and 
Ichiro Takeuchi\thanks{Corresponding author} \\
Nagoya Institute of Technology \\
\texttt{takeuchi.ichiro@nitech.ac.jp} \\
}

\begin{document}

\maketitle

 \vspace{-.1in}
\begin{abstract}
 Discovering statistically significant patterns from databases is an important challenging problem. 
The main obstacle of this problem is in the difficulty of taking into account the selection bias, i.e., the bias arising from the fact that patterns are selected from extremely large number of candidates in databases. 
In this paper, we introduce a new approach for predictive pattern mining problems that can address the selection bias issue. 
Our approach is built on a recently popularized statistical inference framework called \emph{selective inference}. 
In selective inference, statistical inferences (such as statistical hypothesis testing) are conducted based on sampling distributions conditional on a selection event. 
If the selection event is characterized in a tractable way, statistical inferences can be made without minding selection bias issue. 
However, in pattern mining problems, it is difficult to characterize the entire selection process of mining algorithms. 
Our main contribution in this paper is to solve this challenging problem for a class of predictive pattern mining problems by introducing a novel algorithmic framework. 
We demonstrate that our approach is useful for finding statistically significant patterns from databases. 

 \begin{center}
 {\bf Keywords}
 \end{center}
 Statistically-sound data mining;
Predictive pattern mining;
Selective inference;
Statistical hypothesis testing

 \vspace{.1in}
\end{abstract}

\section{Introduction}
\label{sec:intro}
Discovering statistically reliable patterns 
from 
databases
is an important challenging problem. 
This problem is
sometimes
referred to as
\emph{statistically sound pattern discovery}
\cite{hamalainen2014statistically,webb2007discovering}.
In this paper
we introduce a new \emph{statistically sound} approach 
for predictive pattern mining
\cite{fan2008direct,novak2009supervised,zimmermann2014supervised}.
Although the main goal of predictive pattern mining is
to discover patterns
whose occurrences are highly associated
with the response,
it is often desirable
to additionally provide
the statistical significance of the association
for each of the discovered patterns
(e.g., in the form of $p$-values).
However,
properly evaluating the statistical significance of pattern mining results is quite challenging
because the \emph{selection effect} of the mining process
must be taken into account.
Noting that
predictive pattern mining algorithms are designed
to \emph{select} patterns
which are more associated with the response than other patterns in the database,
even if all the patterns in the database have no true associations, 
the discovered patterns would have apparent spurious associations by the \emph{selection effect}.
Such a distortion of statistical analysis is often referred to as
\emph{selection bias} \cite{heckman1979sample}.
\figurename~\ref{fig:selection-bias}
is a simple illustration of selection bias.

\begin{figure}[!h]
 \centering
 \begin{tabular}{cccccc}
  \includegraphics[width=0.13\linewidth]{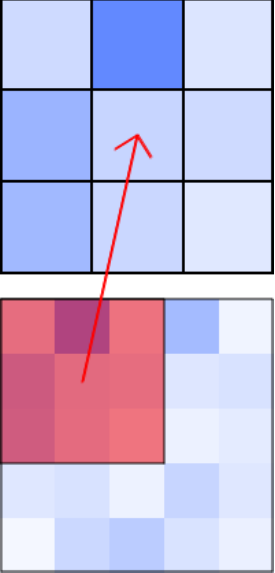} \!\!\!\! & \!\!\!\!
  \includegraphics[width=0.13\linewidth]{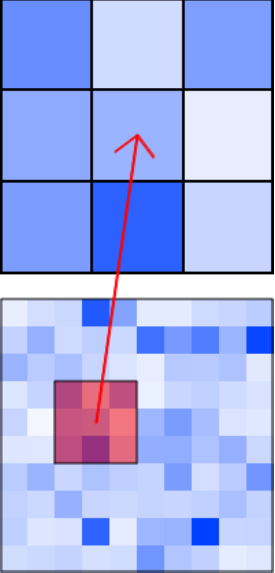} \!\!\!\! & \!\!\!\!
  \includegraphics[width=0.13\linewidth]{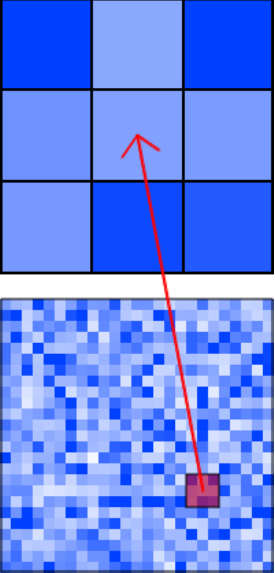} \!\!\!\! & \!\!\!\!
  \includegraphics[width=0.13\linewidth]{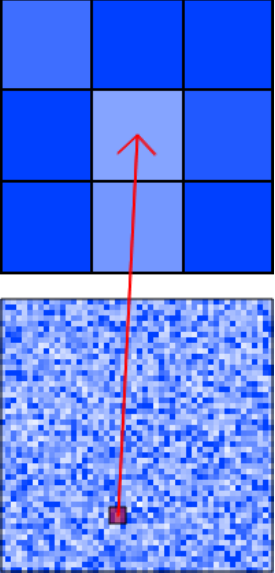} \!\!\!\! & \!\!\!\! 
  \includegraphics[width=0.13\linewidth]{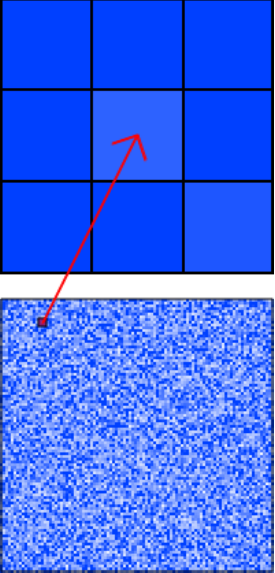} \!\!\!\! & \!\!\!\!
  \includegraphics[width=0.098\linewidth]{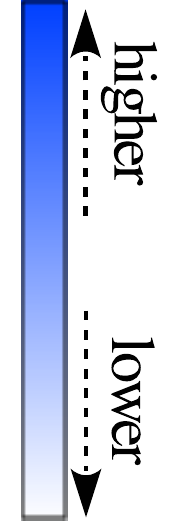} \\
  5$\times$5 & 10$\times$10 & 25$\times$25 & 50$\times$50 & \!\!\!\! 100$\times$100 \\
 \end{tabular}
 \includegraphics[width=0.60\linewidth]{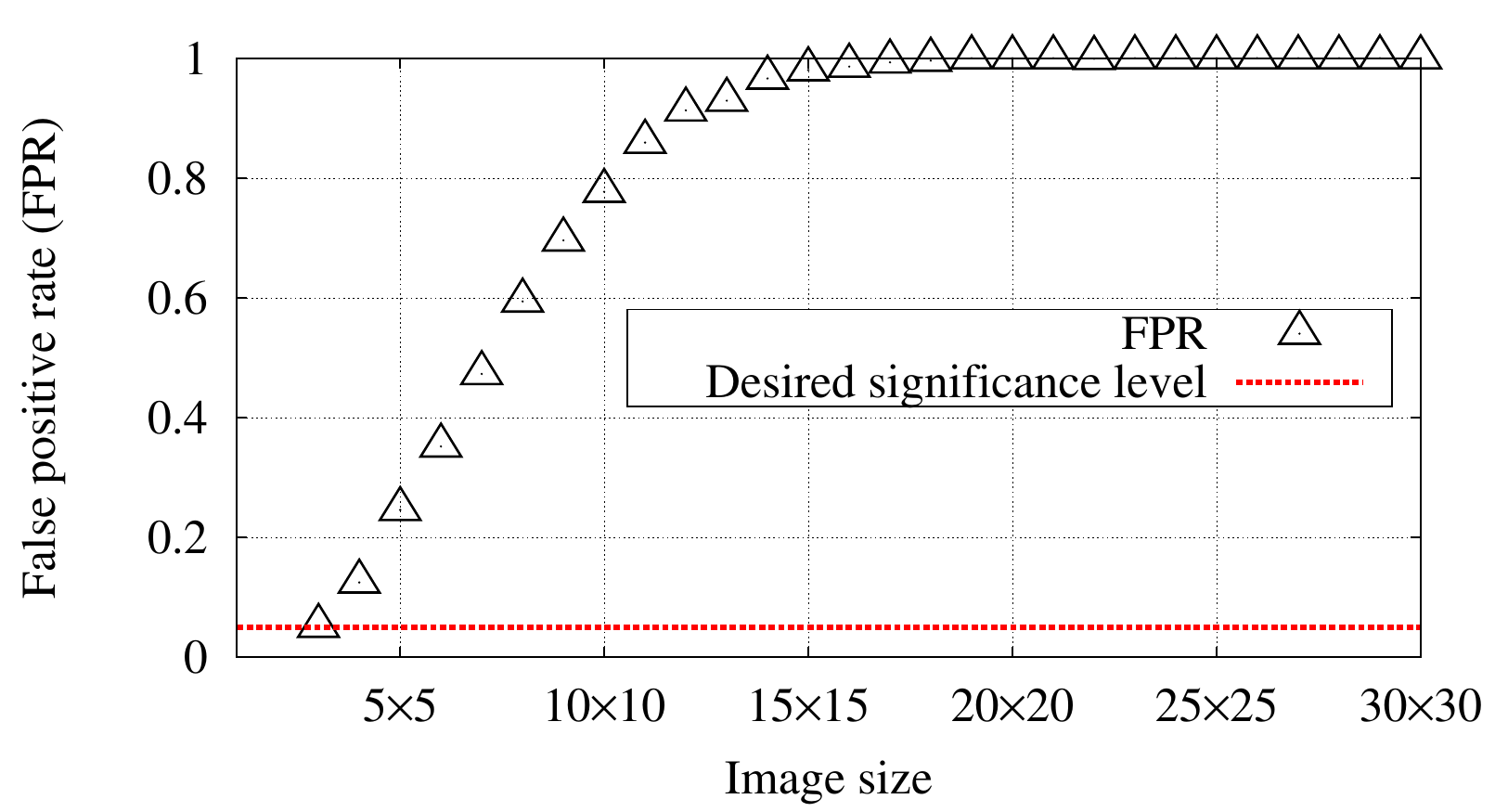}
 \caption{
 A simple demonstration of selection bias. 
 Here, 
 we randomly generated 
 $3 \times 3$,
 $\ldots$, 
 $100 \times 100$
 pixel images
 where the color of each pixel represents the value generated from 
 $N(0, 1^2)$.
 Then, 
 we selected the ``most blue'' 
 $3 \times 3$
 pattern 
 from each of these images. 
 We observe that 
 the selected pattern looks ``significantly blue'' 
 when it is selected from a large image,
 although it is merely spurious result due to the selection effect. 
 The bottom plot shows the frequencies of the false positive findings for various sizes of images
 obtained by applying naive statistical test
 for checking the statistical significance of the average value of the selected $3 \times 3$ pattern. 
 The false positive rates are far greater than the desired 5\% level 
 especially when the pattern is selected from large images. 
 In this paper,
 we introduce a novel approach that can address the selection bias issue
 for a class of predictive pattern mining problems. 
 }
 \label{fig:selection-bias}
\end{figure}

In this paper, we introduce a new approach for solving the selection bias issue for a class of predictive pattern mining problems.
Our new approach is built on a framework called \emph{selective inference} (see, e.g., \cite{taylor2015statistical}).
The main idea of selective inference is that, by considering a sample space conditional on a particular selection event, we do not have to mind the bias stemming from the selection event. 
%
In the context of pattern mining,
it roughly indicates that, if we make a statistical inference (computing $p$-values or confidence intervals etc.) based on a sampling distribution under the condition that a particular set of patterns are discovered, the selection bias of the mining algorithm could be removed. 

Although
the concept of selective inference
has long been discussed in the statistics community,
no practical selective inference frameworks have been 
developed until very recently \cite{berk2013valid}.
The difficulty of selective inference lies in the requirement that 
we must be able to derive the sampling distribution of the test statistic under the condition that the selection event actually takes place.
Although deriving such a sampling distribution is generally intractable,
Lee et al. \cite{lee2013exact} recently proposed a practical selective inference procedure for feature selection problems in linear models.
Specifically,
they provided a method
for computing the sampling distributions
of the selected linear model parameters
under the condition that
a particular set of features are selected
by using a certain class of feature selection algorithms. 

Our main contribution in this paper
is to extend the idea of Lee et al~\cite{lee2013exact},
and develop a selective inference procedure
for a class of predictive pattern mining problems.
We develop a novel method
for computing the exact sampling distribution of a relevant test statistic on the discovered patterns 
under the condition that
those patterns are discovered 
by using the mining algorithm.
We note that
this extension is non-trivial
because we need to take into account extremely large number of patterns
in the database.
%
For circumventing this computational issue,
we consider a tree structure among patterns and derive a novel pruning condition
that enables us to efficiently identify a set of patterns 
which have no effect on the sampling distribution. 
To the best of our knowledge, 
this paper is the first to address selection bias issue in pattern mining via selective inference framework.
The above pruning rule enables us to develop a practical selective inference framework 
that can be applied to a class of predictive pattern mining problems
in which extremely large number of patterns are involved. 

\subsection{Related approaches}
\label{subsec:related-approaches}
In most existing pattern mining procedures,
the reliability of the discovered patterns are quantified by non-statistical measures such as
\emph{support}, 
\emph{confidence},
\emph{lift}
or
\emph{leverage}
\cite{agrawal1993mining}. 
These non-statistical measures are easy to interpret and would be sufficient for some applications.
However,
when the data is noisy and considered to be a random sample from the population of interest,
it is desired to provide statistical significance measures such as 
$p$-values or confidence intervals
for each of the discovered patterns.
Although
several researchers in data mining community studied how to compute statistical significances of the discovered patterns 
\cite{bay2001detecting,yan2008mining,hamalainen2010statapriori,arora2014mining},
the reported $p$-values in these studies are biased 
in the sense that the selection effect of the mining algorithms are not taken into account
(unless a multiple testing correction procedure is applied to these $p$-values afterward). 

In machine learning community, the most common approach for dealing with selection bias is \emph{data splitting}.
In data splitting, the dataset is divided into two disjoint sets.
One of them is used for pattern discovery and the other is used for statistical inference.
Since the inference phase is made independently of the discovery phase, we do not have to care about the selection effect.
An obvious drawback of data splitting is that the \emph{powers} are low both in discovery and inference phases.
Since only a part of the dataset can be used for mining, the risk of failing to discover truly associated patterns would increase.
Similarly, the power of statistical inference (i.e., the probability of true positive finding) would decrease because the inference is made with a smaller dataset.
In addition, it is quite annoying that different patterns might be discovered if the dataset is split differently.
It is important to note that data splitting is also regarded as a selective inference because the inference is made only for the discovered patterns in the discovery phase, and the other undiscovered patterns are ignored. 

In statistics community, \emph{multiple testing correction (MTC)} has been used for addressing selection bias issue 
\cite{shaffer1995multiple,jensen2000multiple}.
MTC methods have been developed for simultaneously control the false positive errors of multiple hypothesis tests (which is sometimes called \emph{simultaneous inference}).
For example, the most common measure for multiple hypothesis testing is \emph{family-wise error (FWE)}, the probability of finding one or more false positives in the multiple tests.
If a MTC method assures FWE control, then the method is also valid for selection bias correction in the sense that the probability of false positive finding can be smaller than the specified significance level $\alpha$.
A notorious drawback of MTC is that they are highly conservative when the number of tests is large, meaning that the power of inference is very low.
For example, in Bonferroni correction method, one can declare a pattern to be positive only when its nominal $p$-value is smaller than $\alpha/J$, where $J$ is the number of all possible patterns in the database \cite{webb2007discovering}~\footnote{
Recently, Terada et al. \cite{terada2013statistical} pointed out in pattern mining context that the denominator of the multiple testing correction can be smaller than $J$ for a certain type of statistical inferences (such as Fisher exact test) by using an idea by Tarone \cite{tarone1990modified},
and several subsequent works in the same direction have been presented
\cite{terada2013fast,minato2014fast,sugiyama2015significant,lopez2015fast}.
}.
Since the number of tests (i.e., the number of all possible patterns $J$) is extremely large, the use of a multiple testing correction usually results in very few significant pattern findings. 

If we use proper selective inference method, 
the corrected $p$-values (called \emph{selective $p$-values} hereafter) of the discovered patterns can be regarded as nominal $p$-values just like they were obtained without selection.
For example,
if we want to control FWE within the discovered patterns, 
we can use Bonferroni correction
just like we only had the discovered patterns from the beginning,
i.e., we declare a pattern to be positive if its selective $p$-value is less than $\alpha/k$ where $k$ is the number of the discovered patterns.
It is interesting to note that,
when $k = J$,
i.e.,
when all the $J$ patterns in the database are \emph{discovered},
the selective inference followed by Bonferroni correction approach 
coincides with the simultaneous inference in the previous paragraph.
In many pattern mining tasks,
simultaneous inference would not be necessary and selective inference would be sufficient 
because we are only interested in the discovered patterns, 
and do not care about the other patterns in the database \cite{benjamini2010simultaneous}. 
In \cite{webb2007discovering},
the author suggested 
to use data splitting approach at first,
and then apply statistical inference with Bonferroni correction
for controlling FWE 
within the discovered patterns.
His approach is similar in spirit with the above selective inference followed by Bonferroni correction approach.

\subsection{Notation and outline}
\label{subsec:notation-outline}
We use the following notations in the remainder. 
For any natural number $n$,
we define
$[n] := \{1, \ldots, n\}$. 
A vector and a matrix is denoted 
such as
$\bm v \in \RR^n$
and 
$M \in \RR^{n \times m}$,
respectively.
%
%
The index function is written as
$\one\{z\}$
which returns $1$ if $z$ is true, and $0$ otherwise. 
The sign function is written as
${\rm sgn}(z)$
which returns $1$ if $z \ge 0$, and $-1$ otherwise.
An $n \times n$ identity matrix is denoted as $I_n$.
%

Here is the outline of the paper. 
\S\ref{sec:preliminary}
presents
problem formulation,
illustrative example,
formal description of selective inference,
and
a brief review of recent selective inference literature.
\S\ref{sec:SI-for-PM}
describes
our main contribution,
where 
we develop a method
that enables selective inference
for a class of discriminative pattern mining problems.
\S\ref{sec:extensions-generalizations}
discusses
extensions and generalizations. 
\S\ref{sec:experiments}
covers numerical experiments
for demonstrating the advantage of selective inference framework
in the context of pattern discovery.
\S\ref{sec:conclusion} 
concludes the paper.

\section{Preliminaries}
\label{sec:preliminary}
In this section,
we first formulate the problem considered in this paper.
Although the selective inference can be similarly applied to
wider class of pattern mining problems
than we consider here,
we study a specific predictive item-set mining problem for concreteness.
Extensions and generalizations are discussed in \S\ref{sec:extensions-generalizations}. 
After presenting a simple illustrative example
in \S\ref{subsec:illustrative-example},
we formally describe selective inference framework
and 
explain why it can be used for addressing selection bias problems in
\S\ref{subsec:selective-inference}.
Finally, 
we review a recent result on selective inference 
by Lee et al.\cite{lee2013exact},
which is the core basis of our main contribution in \S\ref{sec:SI-for-PM}. 

\subsection{Problem statement}
\label{subsec:problem-statement}
We study predictive item-set mining problems
with continuous responses
\cite{deshpande2005frequent,saigo2008partial,saigo2009gboost,ketkar2009gregress}. 
Let us consider a database with $n$ transactions, 
which we denote as
$D := \{(T_i, y_i)\}_{i \in [n]}$. 
Each transaction consists of a subset of $d$ binary items 
$T_i \subseteq T := \{i_1, \ldots, i_d\}$
and a response
$y_i \in \RR$,
where we assume that the latter is centered so that
$\sum_{i \in [n]} y_i = 0$.
We sometimes use a compact notation $D = (\cT, \bm y)$ 
where
$\cT := \{T_i\}_{i \in [n]}$
and
$\bm y := [y_1, \ldots, y_n]^\top \in \RR^n$. 
%
%
%
%
We sometimes restrict our attention on item-sets of the sizes no greater than $r$. 
The set of all those patterns is denoted as
$\cJ := \{t \mid t \in 2^T, |t| \le r\}$, 
its size as
$J := |\cJ| = \sum_{\rho \in [r]} {d \choose \rho}$,
and 
each pattern in $\cJ$ as 
$t_1, \ldots, t_{J} \in \cJ$,
where
$2^T$ is the power set of $T$. 
Similarly,
for each transaction,
the set of patterns in $T_i$ of the sizes no greater than $r$
is denoted as
$\cJ_i$. 
For representing whether each pattern in $\cJ$ is included in a transaction, 
we define 
\begin{align}
\label{eq:occurrence_element}
 \tau_{i, j} := \mycase{
1 & \text{if } t_j \in \cJ_i, \\
0 & \text{if } t_j \notin \cJ_i,
}
\end{align}
for $(i, j) \in [n] \times [J]$.
A vector notation 
$\bm \tau_j := [\tau_{1,j}, \ldots, \tau_{n, j}]^\top \in \{0, 1\}^n$
is used for representing the occurrence of the $j$-th pattern. 

Consider the following concrete example for intuitive understanding of our notations:
\begin{align*}
D = \{
(\{{\rm A, B, C}\}, y_1),
(\{{\rm A, C}\}, y_2),
(\{{\rm B}\}, y_3)
\},
\end{align*}
where we have $n=3$ transactions and $d=3$ items A, B and C.
If we set $r = 2$, $J=5$ patterns\footnote{Note that we do not consider an empty set as a pattern. } are 
\begin{align*}
\begin{tabular}{rcccccl}
 $\cJ$ = \{ \!\!\!\!&\!\!\!\! \{A\}, \!\!&\!\! \{B\}, \!\!&\!\! \{C\}, \!\!&\!\! \{A,B\}, \!\!&\!\! \{A,C\} \!\!\!\!&\!\!\!\! \}. \\
 \phantom{$\cJ$ = } & $\downarrow$ & $\downarrow$ & $\downarrow$ & $\downarrow$ & $\downarrow$ & \\
 \phantom{$\cJ$ = } & $t_1$ &$t_2$ &$t_3$ &$t_4$ &$t_5$ &
\end{tabular}
\end{align*}
Similarly, 
the set of patterns for each transaction are 
\begin{align*}
 \cJ_1 &= \left\{\rm \{A\}, \{B\}, \{C\}, \{A,B\}, \{A,C\} \right\}, \\
 \cJ_2 &= \left\{\rm \{A\}, \{C\}, \{A,C\} \right\}, \\
 \cJ_3 &= \left\{\rm \{B\}\right\}.
\end{align*}
Alternatively,
the occurrence of patterns are represented by the following $n$-by-$J$ matrix information whose $(i, j)$-th element is $\tau_{i, j}$: 
\begin{align*}
 \mtx{ccccc}{
 1 & 1 & 1 & 1 & 1 \\
 1 & 0 & 1 & 0 & 1 \\
 0 & 1 & 0 & 0 & 0 \\
 }.
\end{align*}
The occurrence of each of the $J=5$ patterns is the column of the above matrix, i.e.,
\begin{align*}
 \bm \tau_1 = \mtx{c}{1 \\ 1 \\ 0},
 \bm \tau_2 = \mtx{c}{1 \\ 0 \\ 1},
 \bm \tau_3 = \mtx{c}{1 \\ 1 \\ 0},
 \bm \tau_4 = \mtx{c}{1 \\ 0 \\ 0},
 \bm \tau_5 = \mtx{c}{1 \\ 1 \\ 0}.
\end{align*}

In the statistical inference framework we discuss here, 
we assume that
the response $y_i$ is a sample from a Normal distribution
$N(\mu(T_i), \sigma^2)$,
where
$\mu(T_i)$
is the unknown mean
that possibly depends on the occurrence of patterns in $T_i$, 
and $\sigma^2$ is the known variance. 
Assuming the homoscedasticity and independence, 
the statistical model on which the inference is made is written as 
\begin{align}
 \label{eq:stat-model}
 \bm y \sim N(\bm \mu(\cT), \sigma^2 I_n), 
\end{align}
where
$\bm \mu(\cT) := [\mu(T_1), \ldots, \mu(T_n)]^\top \in \RR^n$. 

The goal of the problem we consider here is to discover patterns that are statistically significantly associated with the response. 
For each pattern 
$t_j \in \cJ$, 
we define a statistic 
$s_j := \bm \tau_j^\top \bm y$
for
$j \in [J]$ 
in order to quantify the strength of the association with the response.
Noting that $\{y_i\}_{i \in [n]}$ are centered, 
the statistic
$s_j$
would have positive (resp. negative) values 
when the occurrence of the pattern $t_j$
is positively (resp. negatively) associated with the response. 

For concreteness,
we consider pattern mining algorithms for discovering the top $k$ patterns 
based on the statistic
$\{s_j\}_{j \in [J]}$. 
%
%
We denote the set of indices of those $k$ discovered patterns as $\cK \subset [J]$,
i.e.,
$|\cK| = k$. 
The goal of this paper is to introduce a procedure 
for providing the statistical significances of the associations in the form of $p$-values 
for those $k$ discovered patterns in $\cK$.

\subsection{An illustrative example}
\label{subsec:illustrative-example}
We illustrate basic concepts of selective inference by a toy example with $n=2$ transactions and $d=2$ items. 
Consider a database $D := \{(\{i_1\}, -1.5), (\{i_2\}, 1.8)\}$.
Since $d=2$, we have $2^2 - 1 = 3$ patterns: $t_1 := \{i_1\}$, $t_2 := \{i_2\}$ and $t_3 := \{i_1, i_2\}$, and the occurrence vectors of these three patterns are $\bm \tau_1 = [1, 0]^\top$, $\bm \tau_2 = [0, 1]^\top$ and $\bm \tau_3 = [0, 0]^\top$.
Suppose that we select only $k=1$ pattern whose association $s_j = \bm \tau_j^\top \bm y, j \in \{1, 2, 3\}$, is greatest.
Since $s_1 = y_1 = -1.5$, $s_2 = y_2 = 1.8$ and $s_3 = 0.0$, the second pattern $t_2$ would be selected here. 

Consider a null hypothesis $H_0$ that $\bm y = [y_1, y_2]^\top$ is from $N(\bm 0, I_2)$.
In naive statistical inference, under $H_0$, the $p$-value of the observed $s_2 = 1.8$ is given by
\begin{align}
 \label{eq:naive-p}
 p = {\rm Prob}(s_2 > 1.8 \mid y_1 = -1.5, H_0) \simeq 0.036 < 0.05, 
\end{align}
meaning that one would conclude that the association of the pattern $t_2$ is significant at $\alpha=0.05$ level.
In selective inference, the statistical significance is evaluated conditional on the selection event that the pattern $t_2$ is selected.
Thus, the selective $p$-value is given by
\begin{align}
 \label{eq:selective-p}
 p = {\rm Prob}(s_2 > 1.8 \mid s_2 = \max\{s_1, s_2, s_3\}, y_1 = -1.5, H_0) \simeq 0.072 > 0.05, 
\end{align}
meaning that one would conclude that the association of the pattern $t_2$ is NOT significant at $\alpha=0.05$ level if we consider the fact that $t_2$ was selected. 

In order to compute selective $p$-values in the form of \eq{eq:selective-p}, we need to characterize the condition $s_2 = \max\{s_1, s_2, s_3\}$ in a tractable way.
In this extremely simple toy example, the condition can be simply written as
\begin{align}
\label{eq:discovered-pattern2}
\bm \tau_2^\top \bm y \ge \bm \tau_1^\top \bm y, \bm \tau_2^\top \bm y \ge \bm \tau_3^\top \bm y ~\Leftrightarrow~ y_2 \ge y_1, y_2 \ge 0. 
\end{align}
It means that the conditional probability in \eq{eq:selective-p} is rephrased as ${\rm Prob}(s_2 \mid y_2 \ge y_1, y_2 \ge 0, y_1 = -1.5, H_0)$. 

\figurename~\ref{fig:toy2d_space} shows the two dimensional sample space of $\bm y = [y_1, y_2]^\top$, where the space is divided into three regions depending on which of the three patterns $t_1$, $t_2$ or $t_3$ would be selected.
The problem of computing the conditional probability in \eq{eq:selective-p} can be interpreted as the problem of computing the probability of $s_2$ conditional on an event that $\bm y$ is observed somewhere in the pink region in \figurename~\ref{fig:toy2d_space}. 
The figure also shows \emph{critical regions} in which $p$-values in \eq{eq:naive-p} or \eq{eq:selective-p} are smaller than 0.05.
In naive inference, $s_2$ is declared to be significantly large if it is greater than $\Phi(0.95)$,
where $\Phi$ is the cumulative distribution function of $N(0, 1^2)$. 
On the other hand, in selective inference, $s_2$ is declared to be significantly large if it is large enough even if we take into account the fact that $s_2$ is greater than $s_1$ and $s_3$. 
\figurename~\ref{fig:toy2d_density} shows the naive sampling distribution in \eq{eq:naive-p} and the selective sampling distribution in \eq{eq:selective-p}. 
The critical region and the sampling distribution of selective inference in {\figurename}s~\ref{fig:toy2d_space} and \ref{fig:toy2d_density} are obtained by using the framework we discuss later.

\begin{figure}[!ht]
 \centering
 \includegraphics[width=0.70\linewidth]{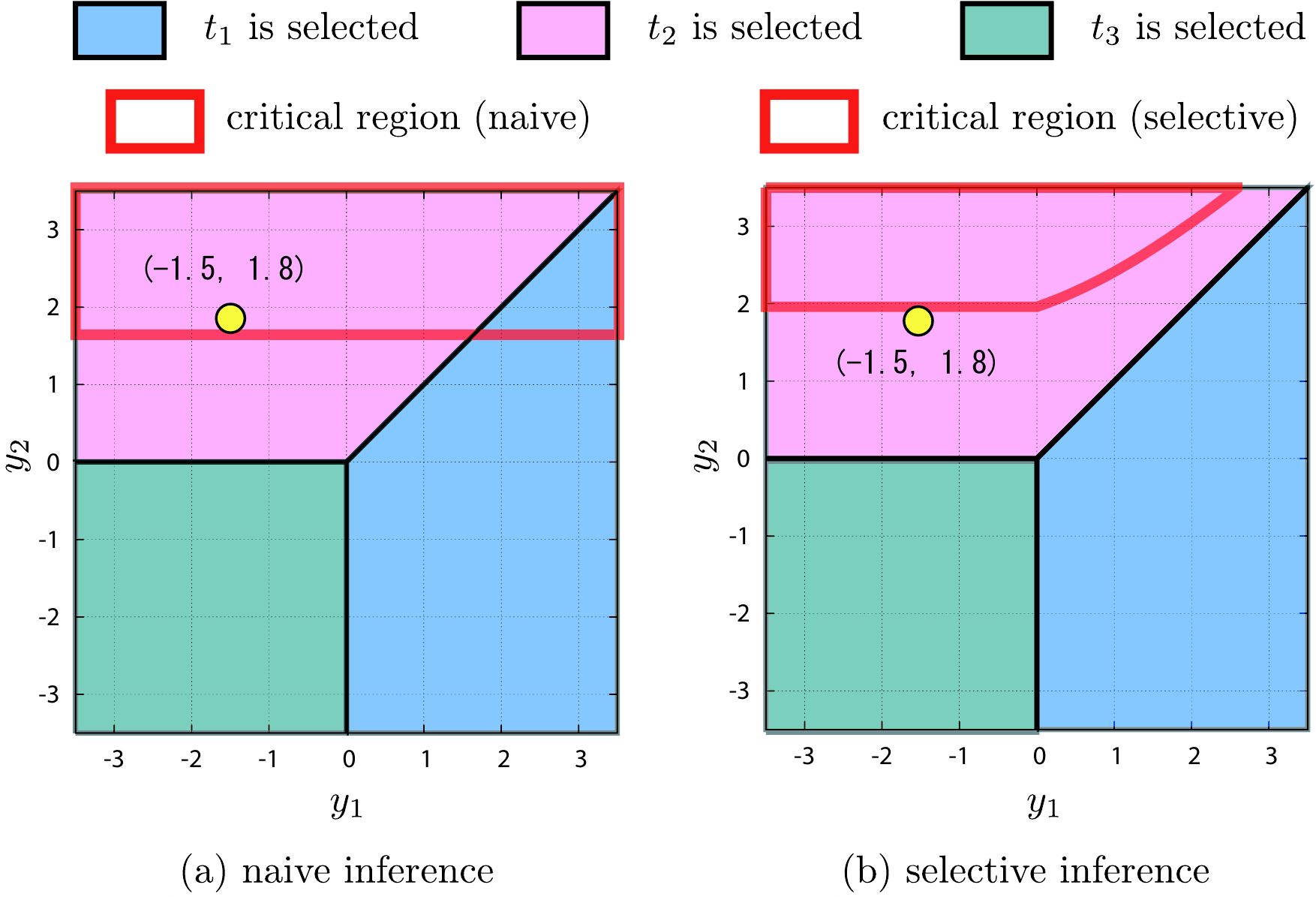}
 \caption{
 Two-dimensional sample space in the toy example where the observation $(y_1, y_2) = (-1.5, 1.8)$ is shown by yellow circle. The space is divided into three regions depending on which of $t_1$, $t_2$ and $t_3$ is selected. Critical regions of the naive inference (left) and the selective inference (right) are shown.}
 \label{fig:toy2d_space}
\end{figure}

\begin{figure}[!ht]
 \centering
 \includegraphics[width=0.7\linewidth]{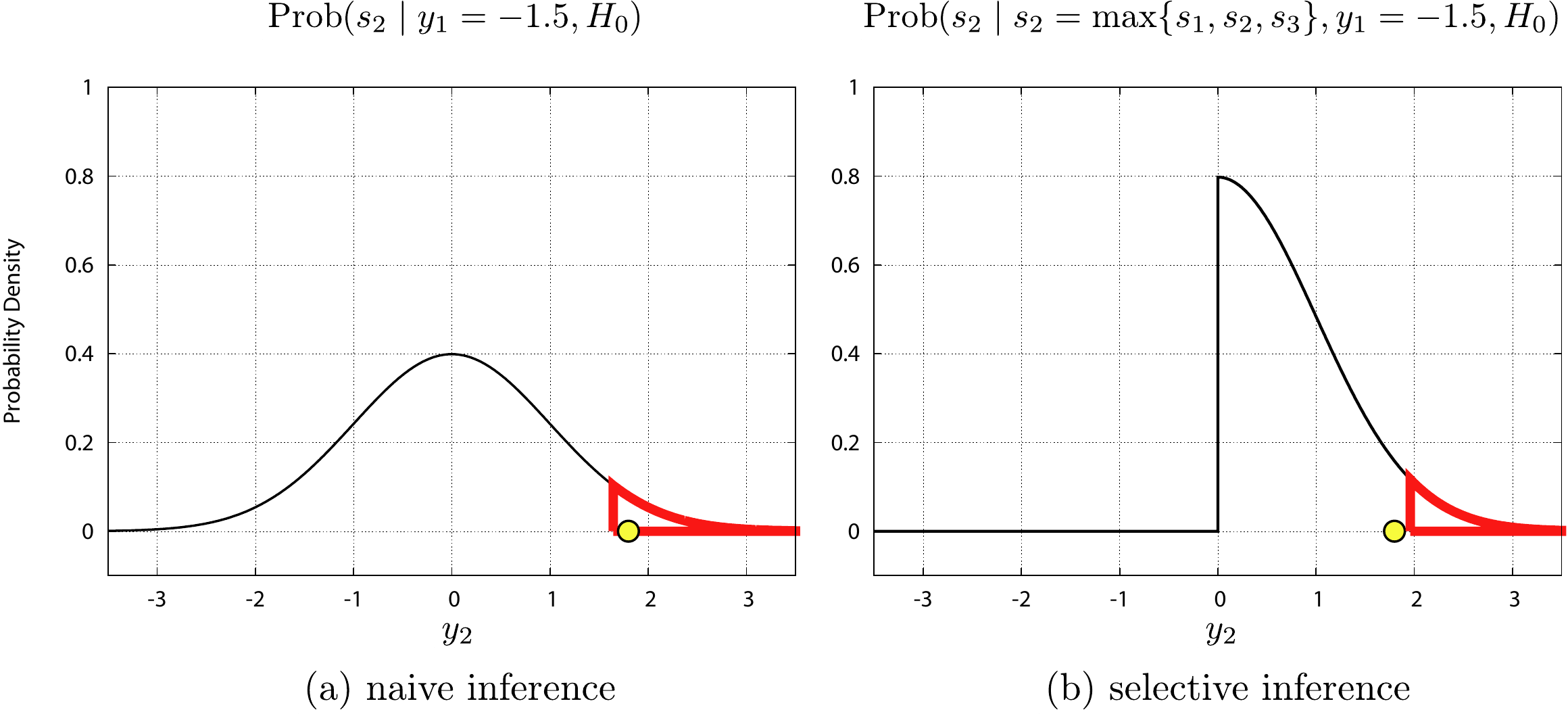}
 \caption{
 Naive sampling distribution (left) and selective sampling distribution (right) of the test statistic $s_2$ in the toy example. 
 The latter is a truncated Normal distribution because it is defined only in the region where $t_2$ is selected (the pink region in \figurename~\ref{fig:toy2d_space}).
 Critical regions and the observation ($y_2 = 1.8$) are shown similarly as in \figurename~\ref{fig:toy2d_space}. 
 }
 \label{fig:toy2d_density}
\end{figure}

\subsection{Selective inference}
\label{subsec:selective-inference}
In this subsection
we formally present
selective inference framework
in the context of
predictive pattern mining problems.
%
See 
\cite{fithian2014optimal}
for a general comprehensive formulation of selective inference framework. 
An inference on  
$s_j = \bm \tau_j^\top \bm y$
is made conditional on its orthogonal component 
in the sample space 
(as the inference  is conditioned on $y_1$ in the toy example in \S\ref{subsec:illustrative-example}). 
We denote the event that the orthogonal component is $\bm w \in \RR^n$ as $\cW(\bm y) = \bm w$. 

We consider the following two-phase procedure:
\begin{enumerate}
 \item 
       {\bf Discovery phase}: 
       Discover a set of patterns
      $\{t_j\}_{j \in \cK}$
      by applying a pattern mining algorithm $\cA$ to the database $D = (\cT, \bm y)$.
      We denote the discovery phase as $\cK = \cA(\cT, \bm y)$.

 \item 
       {\bf Inference phase}:
       For each discovered pattern
       $t_j, j \in \cK$, 
       compute the statistical significance of the association 
       by using a selective inference 
       conditional on an event that the patterns $\{t_j\}_{j \in \cK}$ are discovered. 
\end{enumerate}

The selective inference is conducted under the statistical model \eq{eq:stat-model}. 
In order to test the association between a discovered pattern $t_j$ and the response $\bm y$, we consider the following null hypothesis:
\begin{align}
 \label{eq:null-hyp-uni}
 H_0: \bm \tau_j^\top \bm y \sim N(0, \sigma^2 \|\bm \tau_j\|_2^2). 
\end{align}
Under $H_0$, we define \emph{the selective $p$-value} as 
\begin{align}
 \label{eq:selective-pvalue-uni}
 p_j^{(\cK)} := {\rm Prob}(\bm \tau_j^\top \bm y> s_j \mid \cK = \cA(\cT, \bm y), \cW(\bm y) = \bm w, H_0),
\end{align}
where the superscript 
$^{(\cK)}$
indicates that the selective $p$-values are defined under the condition that the patterns $\{t_j\}, j \in \cK$ are discovered in the first phase. 


\subsubsection{Properties of selective $p$-values}
\label{subsubsec:propery-selective-p}
%
%
Let us define a test $\phi$ as 
\begin{align}
 \label{eq:selective-test}
 \phi(s_j, \cK) = \mycase{
 \text{negative} & \text{if } p_j^{(\cK)} \ge \alpha, \\
 \text{positive} & \text{if } p_j^{(\cK)} < \alpha.
 }
\end{align}
Then,
the probability of \emph{selective false positive error} can be smaller than the significance level $\alpha$, i.e., 
\begin{align*}
 {\rm Prob}(\phi(s_j, \cK) = \text{positive} \mid \cK = \cA(\cT, \bm y), \cW(\bm y) = \bm w, H_0) < \alpha. 
\end{align*}
%
This can be interpreted that,
when a set of patterns discovered by a mining algorithm $\cA$ is given to a user, and the user wants to judge each of the discovered pattern to be positive or negative, 
the test $\phi$ in \eq{eq:selective-test} allows the user to properly control the frequency of false positive findings.

Furthermore,
if a user wants to control family-wise error of the discovered patterns, 
then
we can apply, e.g., usual Bonferroni correction procedure, to the discovered patterns by regarding the $k$ selective $p$-values as the nominal $p$-values. 
Specifically,
let
${\rm FWE}_j^{(\cK)} := k p_j^{(\cK)}$
for
$j \in \cK$. 
Then,
if we select the subset of the discovered patterns $\cK^\prime$ such that
$\cK^\prime := \{j \in \cK \mid {\rm FWE}_j < \alpha\}$, 
then,
we can guarantee that the probability of finding one or more false positives in $\cK^\prime$ is smaller than $\alpha$. 
We call ${\rm FWE}_j, j \in \cK, $ as \emph{Bonferroni-adjusted selective $p$-values} in \S\ref{sec:experiments}. 

We note that,
if we consider two different cases where different patterns $\cK_a$ and $\cK_b$ are discovered in the first phase, 
even when a pattern $t_j$ is discovered in both cases, 
the two $p$-values
$p_j^{(\cK_a)}$
and 
$p_j^{(\cK_b)}$
have different interpretations and cannot be compared. 
A key idea of selective inference is that the inference is made conditional on a single particular selection event $\cK = \cA(\cT, \bm y)$, and other cases are never considered. 
It is important to remind that the goal of selective inference is not to guarantee the goodness of the mining algorithm in the first phase, but to warrant the validity of the inference in the second phase. 

Another important note is about the null hypothesis
$H_0$
in \eq{eq:null-hyp-uni}.
When we specify the null distribution of the statistic $\bm \tau_j^\top \bm y$, we do not need to specify a null distribution of $\bm y \in \RR^n$ itself. 
In other words,
under any null distributions of $\bm y$ in the form of 
\begin{align}
 \label{eq:null-distribution-of-y}
 \bm y \sim N(\bm \mu(\cT), \sigma^2 I)
 \text{ such that }
 \bm \tau_j^\top \bm \mu(\cT) = 0,
\end{align}
the selective $p$-values in \eq{eq:selective-pvalue-uni} has desired property, 
meaning that we do not need to specify any prior knowledge about the data generating process except \eq{eq:null-distribution-of-y}. 
In the simulation study in \S\ref{subsec:illustrative-example},
the null distribution
$\bm y \sim N(\bm 0, \sigma^2 I)$
is just an instance of a class of distributions in the form of \eq{eq:null-distribution-of-y}.

\subsubsection{How to compute selective $p$-values}
\label{subsubsec:computing-selective-p}
The main technical challenge in selective inference 
is 
how we can compute selective $p$-values 
in the form of \eq{eq:selective-pvalue-uni}. 
To this end, 
we need to characterize the selection event 
$\cK = \cA(\cT, \bm y)$
in a tractable way. 
As in the toy example in \S\ref{subsec:illustrative-example},
a selection event that
a particular set of patterns are discovered
by a mining algorithm 
can be interpreted as an event that 
the response vector 
$\bm y$
is observed within a particular region in the sample space $\RR^n$. 
Denoting such a region as 
$\cR(\cK, \cA, \cT) \subseteq \RR^n$,
the above interpretation is formally stated as 
\begin{align*}
 \cK = \cA(\cT, \bm y)
 ~\Leftrightarrow~
 \bm y \in \cR(\cK, \cA, \cT).
\end{align*}

Recently,
Lee et al.~\cite{lee2013exact}
studied a class of feature selection methods 
in which a selection event can be represented
by a set of linear inequalities in the sample space $\RR^n$,
which they call a \emph{linear selection event}. 
In a linear selection event, 
the region 
$\cR(\cK, \cA, \cT)$ 
is a polyhedron.
The authors in \cite{lee2013exact}
showed that,
when 
$\cR(\cK, \cA, \cT)$
is a polyhedron,
the sampling distribution conditional on the polyhedron is a truncated normal distribution, 
and the truncation points are obtained by solving optimization problems over the polyhedron. 
For a class of feature selection problems considered in \cite{lee2013exact}, 
it is possible to solve the optimization problems, 
and the selective $p$-values can be computed with reasonable computational cost. 

In \S\ref{sec:SI-for-PM}, 
we see that 
an event of selecting the top $k$ patterns 
according to the association scores
$s_j = \bm \tau_j^\top \bm y, j \in [J]$, 
can be also represented as a polyhedron in the sample space $\RR^n$. 
Unfortunately,
however,
the polyhedron is potentially characterized by an extremely large number of linear inequalities, 
and
it turns out to be difficult to solve the optimization problems over the polyhedron 
as is done in \cite{lee2013exact}. 
Our main contribution in this paper is to overcome this difficulty 
by developing a novel algorithm
for efficiently identifying linear inequalities that are guaranteed to be irrelevant to the selective sampling distribution.
After briefly reviewing the result of \cite{lee2013exact} in \S\ref{subsec:polyhedral-lemma},
we present selective inference framework for the pattern mining problems in \S\ref{sec:SI-for-PM}.

\subsection{Polyhedral lemma by Lee et al.~\cite{lee2013exact}}
\label{subsec:polyhedral-lemma}
In this subsection, we summarize the recent result by Lee et al.~\cite{lee2013exact}. 
\begin{lemm}[Polyhedral Lemma \cite{lee2013exact}]
\label{lemm:polyhedral}
Consider a linear selection event that the corresponding region 
$\cR(\cK, \cA, \cT)$
is a polyhedron, 
and denote it as 
${\rm Pol}(\cK, \cA, \cT)$. 
For a statistic
in the form of
$\bm \eta^\top \bm y$
with an arbitrary
$\bm \eta \in \RR^n$,
under a null hypothesis
$H_0: \bm \eta^\top \bm \mu(\cT) = 0$
in the statistical model \eq{eq:stat-model}, 
the sampling distribution of
$\bm \eta^\top \bm y$ 
conditional on a selection event
$\bm y \in {\rm Pol}(\cK, \cA, \cT)$ 
can be written as
\begin{align*}
 {\rm Prob}(\bm \eta^\top \bm y \le s \mid \bm y \in {\rm Pol}(\cK, \cA, \cT), \cW(\bm y) = \bm w, H_0)) 
 \sim F_{0, \sigma^2 \|\bm \eta\|_2^2}^{[L(\bm w), U(\bm w)]}(s)
\end{align*} 
where
$F_{m, s^2}^{[L(\bm w), U(\bm w)]}$
represents the cumulative distribution function (c.d.f.) of the truncated Normal distribution 
which is defined by truncating the c.d.f. of a Normal distribution
 $N(m, s^2)$
 at $[L(\bm w), U(\bm w)]$,
and the truncation points
$L(\bm w)$ and $U(\bm w)$
are given as 
\begin{subequations}
 \label{eq:truncation-points}
 \begin{align}
  \nonumber
  L(\bm w) &:=  \bm \eta^\top \bm y + \theta_{\rm min} \|\bm \eta\|_2^2
  ~\text{ where }~ \\
  &\theta_{\rm min} :=
  \min_{\theta \in \RR} ~ \theta
  ~\text{ s.t. }~ \bm y + \theta \bm \eta \in {\rm Pol}(\cK, \cA, \cT),
  \\
  \nonumber
  U(\bm w) &:=  \bm \eta^\top \bm y + \theta_{\rm max} \|\bm \eta\|_2^2
  ~\text{ where }~ \\
  &\theta_{\rm max} :=
  \max_{\theta \in \RR} ~ \theta
  ~\text{ s.t. }~ \bm y + \theta \bm \eta \in {\rm Pol}(\cK, \cA, \cT).
 \end{align}
\end{subequations}
\end{lemm}

The above lemma tells that
the selective sampling distribution
is defined by considering the frequency property of the statistic 
$\bm \eta^\top \bm y$
within the polyhedron
${\rm Pol}(\cK, \cA, \cT)$,
which can be characterized by solving a minimization and a maximization problems over the polyhedron in \eq{eq:truncation-points}. 
Remembering that $\bm y$ is Normally distributed, 
$\bm \eta^\top \bm y$
is also Normally distributed. 
If we restrict our attention only within the polyhedron ${\rm Pol}(\cK, \cA, \cT)$, 
the distribution of $\bm \eta^\top \bm y$ is a truncated Normal distribution 
in which each truncation point corresponds to one of the boundaries of the polyhedron. 
%
%
See \cite{lee2013exact} for the proof and more detailed implications of Lemma~\ref{lemm:polyhedral}. 

\section{Selective inference for predictive pattern mining}
\label{sec:SI-for-PM}
In this section,
we introduce a selective inference procedure
for the pattern mining problem 
described in the previous section. 
In \S\ref{subsec:pma-lse},
we first present 
the pattern mining method 
for the discovery phase,
and
discuss that 
a selection event
by this algorithm
is characterized 
by a set of linear inequalities in the sample space.
Next, 
in \S\ref{subsec:spc-pm},
we present a novel method 
for the inference phase,
in which we can efficiently handle extremely large number of patterns in the database.

Both methods 
in the two phases 
are developed 
by exploiting anti-monotonicity properties 
defined in the item-set tree structure 
as depicted in
\figurename~\ref{fig:tree}. 
Each node of the tree corresponds to each pattern $t_j$ in $\cJ$,
and 
same index $j \in [J]$
is used for representing a node and the corresponding pattern. 
For each node $j \in [J]$ in the tree, 
we denote the set of its descendant nodes as
$Des(j) := \{\ell \in [J] \mid t_j \subseteq t_{\ell}\}$. 

\begin{figure}[h]
 \begin{center}
  \includegraphics[width=0.5\linewidth]{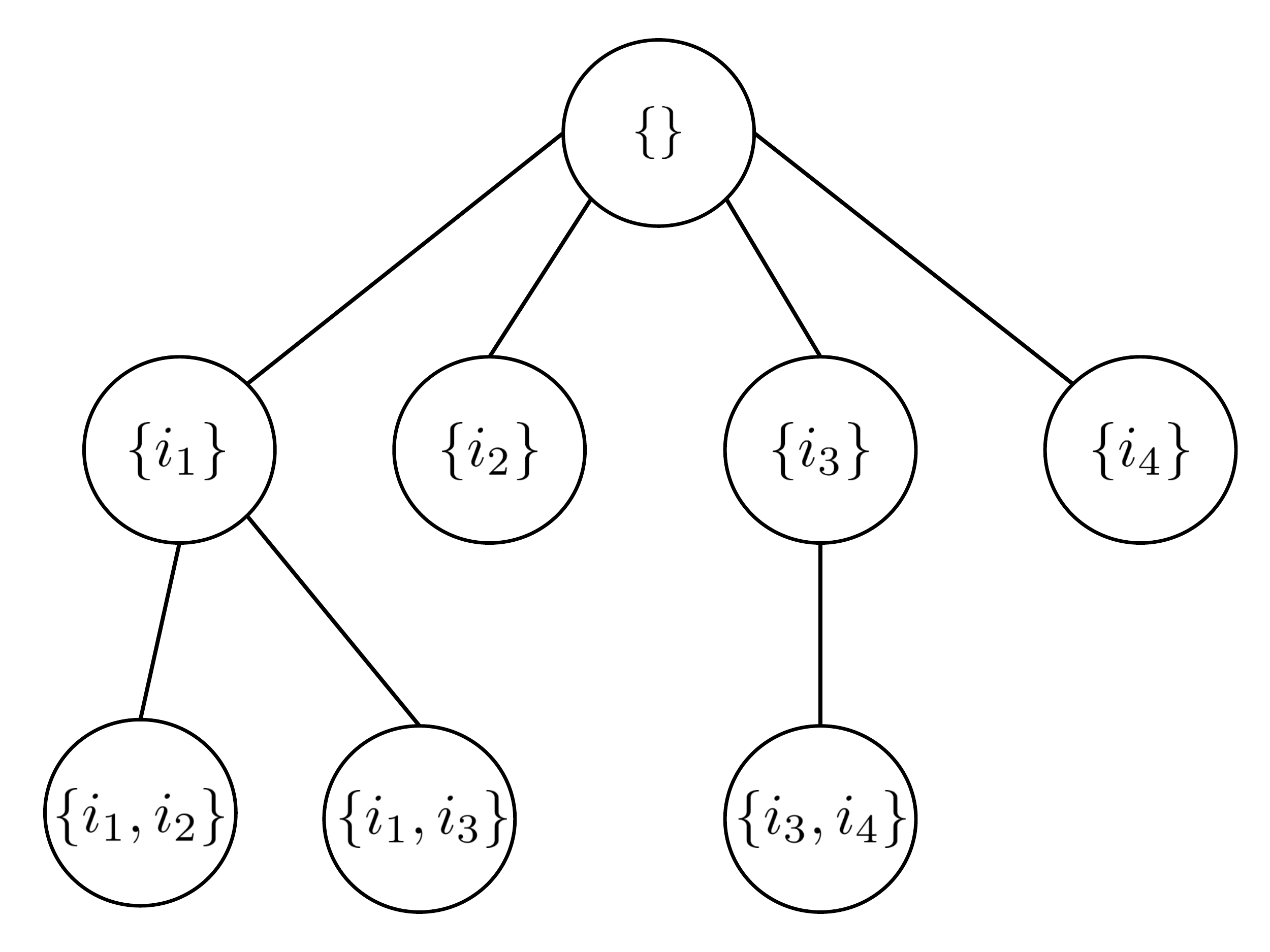}  
  \caption{An illustration of tree structure for item-set mining problems.}
  \label{fig:tree}
 \end{center}
\end{figure}

\subsection{Pattern mining as a linear selection event}
\label{subsec:pma-lse} 
In order to discover the top $k$ associated patterns, 
we develop a method
searching over the item-set tree as depicted in \figurename~\ref{fig:tree}. 
In the search over the tree,
we use the following pruning criterion.
\begin{lemm}
 \label{lemm:discovery-univariate-anti-monotonicity}
 Consider a node $j$ in the tree structure
 as depicted in \figurename~\ref{fig:tree}
 corresponding to a pattern $j \in [J]$. 
 Then,
 for any descendant node $\ell \in Des(j)$, 
 \begin{align}
  \label{eq:discovery-univariate-condition} 
  s_{\ell} \le \sum_{i: y_i > 0} \tau_{i, j} y_i.
 \end{align}
\end{lemm}
\begin{proof}
 Noting that
 $0 \le \tau_{i, \ell} \le \tau_{i, j} \le 1$, 
 \begin{align*}
  s_\ell \!:=\! 
  \bm \tau_{\ell}^\top \bm
  \!=\! \sum_{i : y_i > 0}\! \tau_{i, \ell} y_i \!+\! \sum_{i : y_i < 0}\! \tau_{i, \ell} y_i 
  \le\! \sum_{i : y_i > 0}\! \tau_{i, \ell} y_i
  \le\! \sum_{i : y_i > 0}\! \tau_{i, j} y_i.
 \end{align*}
\end{proof}
We note that 
Lemma~\ref{lemm:discovery-univariate-anti-monotonicity} is not new. 
This simple upper bound has been used in several data mining studies such as \cite{kudo2004application,nakagawaKDD2016submitted}. 
When we search over the tree,
if
the upper bound in \eq{eq:discovery-univariate-condition} 
is smaller than the current $k$-th largest score at a certain node $j$,
then
we can quit searching over its descendant nodes $\ell \in Des(j)$. 

A selection event by the above method can be characterized by a set of linear inequalities in the sample space $\RR^n$.
Noting that a fact that
$k$ patterns
$\{t_j\}_{j \in \cK}$ 
are discovered from
the database indicates that 
their scores
$s_j, j \in \cK$, 
are greater than those of the other non-discovered patterns 
$s_j, j \in [J] \setminus \cK$.
This fact can be simply formulated as
\begin{align}
 \label{eq:linear-event}
 \bm \tau_j^\top \bm y \ge \bm \tau_{j^\prime}^\top \bm y
 ~
 \forall 
 (j, j^\prime)
 \in
 \cK \times \{[J] \setminus \cK\}.
\end{align}
Namely,
a selection event by the above mining method is represented as a polyhedron 
${\rm Pol}(\cK, \cA, \cT)$
defined by
$k(J-k)$
linear inequalities in $\RR^n$. 
It indicates that,
in theory,
we can apply the polyhedral lemma in \S\ref{subsec:polyhedral-lemma}
to this problem. 
In practice,
however,
it is computationally intractable to naively handle all these 
$k(J-k)$
linear inequalities. 

\subsection{Selective $p$-value for pattern mining}
\label{subsec:spc-pm}
The discussion in
\S\ref{subsec:pma-lse}
suggests that
it would be hard to compute selective $p$-values
in the form of
\eq{eq:selective-pvalue-uni}
because the selection event
$\cK = \cA(\cT, \bm y)$
is characterized by extremely large number of patterns in the database.
Our basic idea
for addressing this computational difficulty
is to note that 
most of the patterns in the database
actually do not affect the sampling distribution for the selective inference, 
and a large portion of them can be identified 
by exploiting the anti-monotonicity properties 
in the item-set trees. 

Specifically,
we consider $k$ item-set trees for each of the $k$ discovered patterns. 
Each tree consists of a set of nodes corresponding to each of the non-discovered patterns 
$\{t_{j^\prime}\}_{j^\prime \in [J] \setminus \cK}$.
For a pair
$(j, j^\prime) \in \cK \times \{[J] \setminus \cK\}$, 
the $j^{\prime}$-th node in the $j$-th tree 
corresponds to the linear inequality 
$\bm \tau_j^\top \bm y \ge \bm \tau_{j^\prime}^\top \bm y$ 
in
\eq{eq:linear-event}. 
When we search over these $k$ trees, 
we introduce a novel pruning strategy by deriving a condition such that, 
if the $j^{\prime}$-th node in the $j$-th tree satisfies a certain condition, 
then all the linear inequalities
$\bm \tau_{j}^\top \bm t \ge \bm \tau_{\ell^{\prime}}^\top \bm y$
for
$\ell^\prime \in Des(j^\prime)$ 
can be ignored 
because they are guaranteed to be irrelevant to the sampling distribution
for the selective inference,
where,
with a slight abuse of notation,
$Des(j^\prime) :=
\{\ell^\prime \in \{[J] \setminus  \cK\} \mid t_{j^\prime} \subseteq t_{\ell^\prime}\}$. 
%

\begin{prop} \label{prop:search-over-k-trees}
 For solving the optimization problems in \eq{eq:truncation-points}, 
 consider the problem of searching over
 all the nodes in the $k$ trees,
 and use a notation
 $(j, j^\prime) \in \cK \times \{[J] \setminus \cK\}$
 for representing the
 $j^\prime$-th node in the $j$-th tree. 
 Then,
 the solutions of the optimization problems in \eq{eq:truncation-points} are respectively written as
 \begin{subequations} 
 \begin{align}
  \label{eq:theta_min}
  \theta_{\rm min}
  &=
  \max_{
   \substack{
    (j, j^\prime) \in \cK \times \{[J] \setminus \cK\},\\
    (\bm \tau_{j^\prime} - \bm \tau_{j})^\top \bm \eta < 0
   }
  }
  \frac{(\bm \tau_{j} - \bm \tau_{j^\prime})^\top \bm y}{(\bm \tau_{j^\prime} - \bm \tau_{j})^\top \bm \eta}, \\
  \label{eq:theta_max}
  \theta_{\rm max}
  &=
  \min_{
   \substack{
    (j, j^\prime) \in \cK \times \{[J] \setminus \cK\},\\
    (\bm \tau_{j^\prime} - \bm \tau_{j})^\top \bm \eta > 0
   }
  }
  \frac{(\bm \tau_{j} - \bm \tau_{j^\prime})^\top \bm y}{(\bm \tau_{j^\prime} - \bm \tau_{j})^\top \bm \eta}.
 \end{align}
 \end{subequations}
\end{prop}
The proof of Proposition~\ref{prop:search-over-k-trees} is presented in Appendix.
This proposition indicates that the problem of computing the sampling distribution for the selective inference is reduced to the problem of searching over the $k$ trees. 
In the following theorem,
we introduce a novel pruning condition for making the search efficient. 

\begin{theo}
 \label{theo:pruning_cond}
 Consider a situation that we have already searched over some nodes in some trees, and denote them as
 $\cV \subset \cK \times \{[J] \setminus \cK\}$.
 Furthermore,
 let us write the current estimates of $\theta_{\rm min}$ and $\theta_{\rm max}$ as
 $\hat{\theta}_{\rm min}^{\cV}$ and $\hat{\theta}_{\rm max}^{\cV}$ respectively.

 For any node in any tree $(j, j^\prime)$,
 if either of the following conditions 
 \begin{align}
  \label{eq:theo-main-a1}
  \sum_{i : \eta_i < 0} \tau_{i, j^\prime} \eta_i - \bm \tau_{j}^\top \bm \eta \ge 0,
 \end{align}
 or
 \begin{align}
  \label{eq:theo-main-a2}
  \!\!\!\!\! \!\!\!\!\! 
  \bm \tau_j^\top \bm y - \sum_{y_i > 0} \tau_{i, j^\prime} y_i \ge 0
  \text{  and  }
  \frac{
   \bm \tau_j^\top \bm y - \sum_{y_i > 0} \tau_{i, j^\prime} y_i
  }{
   \sum_{i : \eta_i < 0} \tau_{i, j^\prime} \eta_i - \bm \tau_{j}^\top \bm \eta
  }
  \le
  \hat{\theta}_{\rm min}^{\cV}
 \end{align}
 are satisfied,
 then
 its descendant nodes 
 $(j, \ell^\prime)$
 for 
 $\ell^\prime \in Des(j^\prime)$
 do not affect the solution of \eq{eq:theta_min},
 i.e.,
 $\theta_{(j, \ell^\prime)}$ does not satisfy the constraint in \eq{eq:theta_min}
 or 
 $\theta_{(j, \ell^\prime)}$ is smaller than the current estimate $\hat{\theta}_{\rm min}^{\cV}$. 

 Similarly, for any node in any tree $(j, j^\prime)$,
 if either of the following conditions 
 \begin{align*}
  \sum_{i : \eta_i > 0} \tau_{i, j^\prime} \eta_i - \bm \tau_{j}^\top \bm \eta \le 0
 \end{align*}
 or
 \begin{align*}
  \bm \tau_j^\top \bm y - \sum_{y_i > 0} \tau_{i, j^\prime} y_i \ge 0
  \text{  and  }
  \frac{
   \bm \tau_j^\top \bm y - \sum_{y_i > 0} \tau_{i, j^\prime} y_i
  }{
   \sum_{i : \eta_i > 0} \tau_{i, j^\prime} \eta_i - \bm \tau_{j}^\top \bm \eta
  }
  \ge
  \hat{\theta}_{\rm max}^{\cV}
 \end{align*}
 are satisfied,
 then
 its descendant nodes 
 $(j, \ell^\prime)$
 for 
 $\ell^\prime \in Des(j^\prime)$
 do not affect the solution of \eq{eq:theta_max},
 i.e.,
 $\theta_{(j, \ell^\prime)}$ does not satisfy the constraint in \eq{eq:theta_max}
 or 
 $\theta_{(j, \ell^\prime)}$ is greater than the current estimate $\hat{\theta}_{\rm max}^{\cV}$. 
\end{theo}
The proof of Theorem~\ref{theo:pruning_cond} is presented in Appendix.
This theorem provides explicit pruning conditions in the search process over the $k$ trees, 
and enables selective $p$-value computation
by making good use of the anti-monotonicity properties in the trees
for efficiently identifying the patterns
that are not relevant to the sampling distribution. 

The pruning conditions in Theorem~\ref{theo:pruning_cond} do not depend on specific search strategies over the $k$ trees.
In practice, it is more efficient to search both
$\theta_{\rm min}$
and
$\theta_{\rm max}$
simultaneously.
In this case,
we can develop slightly different pruning conditions 
that can be commonly used for the two search problems. 
Due to the space limitation,
we do not describe the specific implementation of our search strategy.

\section{Extensions}
\label{sec:extensions-generalizations}
So far,
we focus on 
a specific class of pattern mining problems
described in
\S\ref{subsec:problem-statement} 
for concreteness. 
In this section, we discuss extensions.

\subsection{Discovering positive and negative associations simultaneously}
\label{subsec:pos-neg}
Previously,
we have studied
the problem of discovering the top $k$ positively associated patterns
(or the top $k$ negatively associated patterns).
It is often desired to discover the top $k$ associated patterns 
regardless of the signs of associations. 
In this case,
it is natural to select the top $k$ patterns
whose absolute scores 
$|s_j|, j \in [J]$
are greater than the others. 
In this situation,
it is appropriate to make inferences conditional not only on the selected patterns but also on their signs. 
To realize this,
we slightly change the definitions of discovery event and selective $p$-values. 
Let us define 
$\tilde{\cK} := \{(j, {\rm sgn}(s_j))\}_{j \in \cK}$,
i.e.,
the set of the discovered patterns and the signs of the associations,
and write the discovery phase as
$\tilde{\cK} = \cA(\cT, \bm y)$. 
%
Then,
in the inference phase,
we define selective $p$-values
depending on the signs of the associations
in the following way:
\begin{align*}
 p_j^{(\tilde{\cK})} := \mycase{
 {\rm Prob}(\bm \tau_j^\top \bm y > s_j \mid \tilde{\cK} = \cA(\cT, \bm y), \cW(\bm y) = \bm w, H_0) ~~ \text{if } {\rm sgn}(s_j) > 0, \\
 {\rm Prob}(\bm \tau_j^\top \bm y < s_j \mid \tilde{\cK} = \cA(\cT, \bm y), \cW(\bm y) = \bm w, H_0) ~~ \text{if } {\rm sgn}(s_j) < 0.
 }
\end{align*}
This definition
is based on the idea that, 
if a pattern is discovered in the first step because of its high positive (resp. negative) association, 
we would be only interested in testing whether the positive (resp. negative) association is statistically significant or not after correcting the selection bias. 
%
%
By conditioning
not only on the observed discovered patterns
but also on the observed signs of the associations,
the selection event is characterized by
$2k(J-k)$
linear inequalities: 
 $|\bm \tau_{j}^\top \bm y| \ge |\bm \tau_{j^\prime}^\top \bm y|
 ~\Leftrightarrow~
 \left(
 {\rm sgn}(\bm \tau_{j}^\top \bm y) \bm \tau_{j}
 \pm
 \bm \tau_{j^\prime}
 \right)^\top
 \bm y
 \ge
 \zero$
for all
$(j, j^\prime) \in \cK \times [J] \setminus \cK$.

\subsection{Sequential pattern discovery}
\label{subsec:sequential}
If the goal is to discover 
a set of patterns
that are useful 
for predictive modeling, 
it is not appropriate to select patterns
based only on the individual associations with the response. 
In this case, 
we should also consider correlations among the patterns 
because
having multiple highly correlated patterns in predictive models is not very helpful. 
In the context of linear model learning, 
this problem is called feature selection, 
and many feature selection approaches have been studied in the literature (see, e.g., \S3 in \cite{friedman2001elements}). 
Here,
we focus on a sequential pattern discovery approach 
in which
relevant features are sequentially discovered one by one. 
We note that
selective inference framework for sequential feature selection in linear models 
has been already studied in \cite{lee2014exact}.
Our contribution here is
again to extend it
to predictive pattern mining problems 
by overcoming the computational difficulty
in handling extremely large number of patterns
in the database. 

\subsubsection{Discovery phase}
\label{subsubsec:OMP-dis}
Here,
we study a sequential predictive pattern discovery method. 
Let 
$\cK_h := [(1), \ldots, (h)]$
be the sequence of the discovered pattern indices 
from step 1 to step $h$
for $h \in [k]$. 
Before step $h+1$,
we have already discovered $h$ patterns
$\{t_j\}_{j \in \cK_h}$. 
Using these $h$ patterns,
the linear predictive model is written as 
$\sum_{\ell \in [h]} \hat{\beta}_{(\ell)}^{\cK_h} \bm \tau_{(\ell)}$, 
where
the coefficients 
$\{ \hat{\beta}_{(\ell)}^{\cK_h} \}_{\ell \in [h]}$
are estimated by least-squares method.
Denoting
$\Gamma^{\cK_h}$
be
$n \times h$
matrix
whose $\ell$-th column is
$\bm \tau_{(\ell)}$,
the least square estimates are written as 
\begin{align*}
 \hat{\bm \beta}^{\cK_h}
 :=
  [\hat{\beta}_{(1)}^{\cK_h}, \ldots, \hat{\beta}_{(h)}^{\cK_h}]^\top
 =
 (\Gamma^{\cK_h})^+ \bm y,
\end{align*}
where
$(\Gamma^{\cK_h})^+$
is the pseudo-inverse of
$\Gamma^{\cK_h}$.
Then,
at the $h + 1$ step, 
we consider the association between
the residual vector 
$\bm r_h := \bm y - \Gamma^{\cK_h} \hat{\bm \beta}^{\cK_h}$
and 
a pattern
$t_j$ for $j \in [J] \setminus \cK_h$,  
and discover the one that maximizes 
$|\bm r_h^\top \bm \tau_j|$
among the patterns 
$\{t_j\}_{j \in [J] \setminus \cK_{h}}$.
Due to the space limitation,
we do not describe the mining algorithm. 
We can develop it 
by using similar techniques as 
Lemma~\ref{lemm:discovery-univariate-anti-monotonicity}.

In the discovery phase,
we thus consider a selection event that $k$ patterns and their signs are sequentially selected as described above. 
Namely,
the selection event is written as
$\tilde{\cK} = \cA(\cT, \bm y)$ 
with 
$\tilde{\cK} := \{((h), {\rm sgn}(\bm r_h^\top \bm \tau_{(h)})) \}_{h \in [k]}$. 
At each step $h \in [k]$,
an event that the feature
$t_{(h)}$
is discovered is written as
 \begin{align}
 \label{eq:OMP-SE} 
 |\bm r_h^\top \bm \tau_{(h)}| \ge |\bm r_h^\top \bm \tau_{(h^\prime)}|
 ~\Leftrightarrow~
 \left(
 {\rm sgn}(\bm r_{h}^\top \bm \tau_{(h)}) \bm \tau_{(h)}^\top P^{\cK_{h}} 
  \pm \bm \tau_{(h^\prime)}^\top P^{\cK_{h^\prime}} 
 \right)
 \bm y \ge \zero
 \end{align}
for all 
$h^\prime \in [J] \setminus \cK_{h-1} \setminus \{(h)\}$,
where 
$P^{\cK_h} := I_n - (\Gamma^{\cK_h})^+(\Gamma^{\cK_h})^\top$.
%
By combining all the linear selection events in $k$ steps, 
the entire selection event of the above sequential discovery method can be characterized by 
$2\sum_{h \in [k]} (J - h)$
linear inequalities in $\RR^n$. 
It means that, 
in theory, 
we can also apply polyhedral lemma to this sequential discovery method. 
In practice, 
however, 
it is computationally intractable to handle those extremely large number of linear inequalities. 

\subsubsection{Inference phase}
\label{subsubsec:OMP-inf}
In order to quantify the importance of each of the discovered patterns in the linear model,
we make statistical inference
on each least-square coefficient
$\hat{\beta}_{(j)}^{\cK_{j}} = ((\Gamma^{\cK_h})^{+} \bm e_j)^{\top} \bm y$,
$j \in [k]$,
with 
$\bm e_j$
being a $k$-dimensional vector with 1 at the $j$-th element and 0 otherwise. 
The null hypothesis for the $j$-th coefficient is
\begin{align*}
 H_{0,j}: ((\Gamma^{\cK_h})^{+} \bm e_j)^{\top} \bm y \sim N(0, \sigma^2 \bm e_j^\top ((\Gamma^{\cK_k})^{+})^\top (\Gamma^{\cK_k})^{+}  \bm e_j).
\end{align*}
Consider a polytope
${\rm Pol}(\tilde{\cK}_k, \cA, \cT)$
defined by
$2\sum_{h \in [k]} (J - h)$
linear inequalities 
in the form of 
\eq{eq:OMP-SE}.
Then, 
the sampling distribution for the selective inference is a truncated Normal distribution 
whose truncation points are given by solving
a minimization and a maximization problems
over the polyhedron
${\rm Pol}(\tilde{\cK}_k, \cA, \cT)$.
Using Theorem~\ref{theo:pruning_cond},
we can develop a similar algorithm
for efficiently solving these optimization problems. 

\subsection{Mining statistically sound subgraphs}
In this section, we extend the selective inference framework to graph mining problems. 
The goal of graph mining is to extract interesting structures from graph data, and have been demonstrated to be useful for several areas such as biology, chemistry, material science, etc \cite{takigawa2013graph,jiawei2006datamining,lancichinetti2011finding,weill2009development,borgelt2002mining,saigo2009gboost}. 
Here, we use selective inference framework for providing proper statistical significance measures of the extracted subgraphs obtained by graph mining algorithms. 
We use gSpan \cite{yan2002gspan} algorithm for enumerating frequently appeared subgraphs in datasets.
\figurename~\ref{fig:tree_graph} shows an illustration of a tree structure in graph mining problems. 
\begin{figure}[h]
 \begin{center}
  \includegraphics[width=0.5\linewidth]{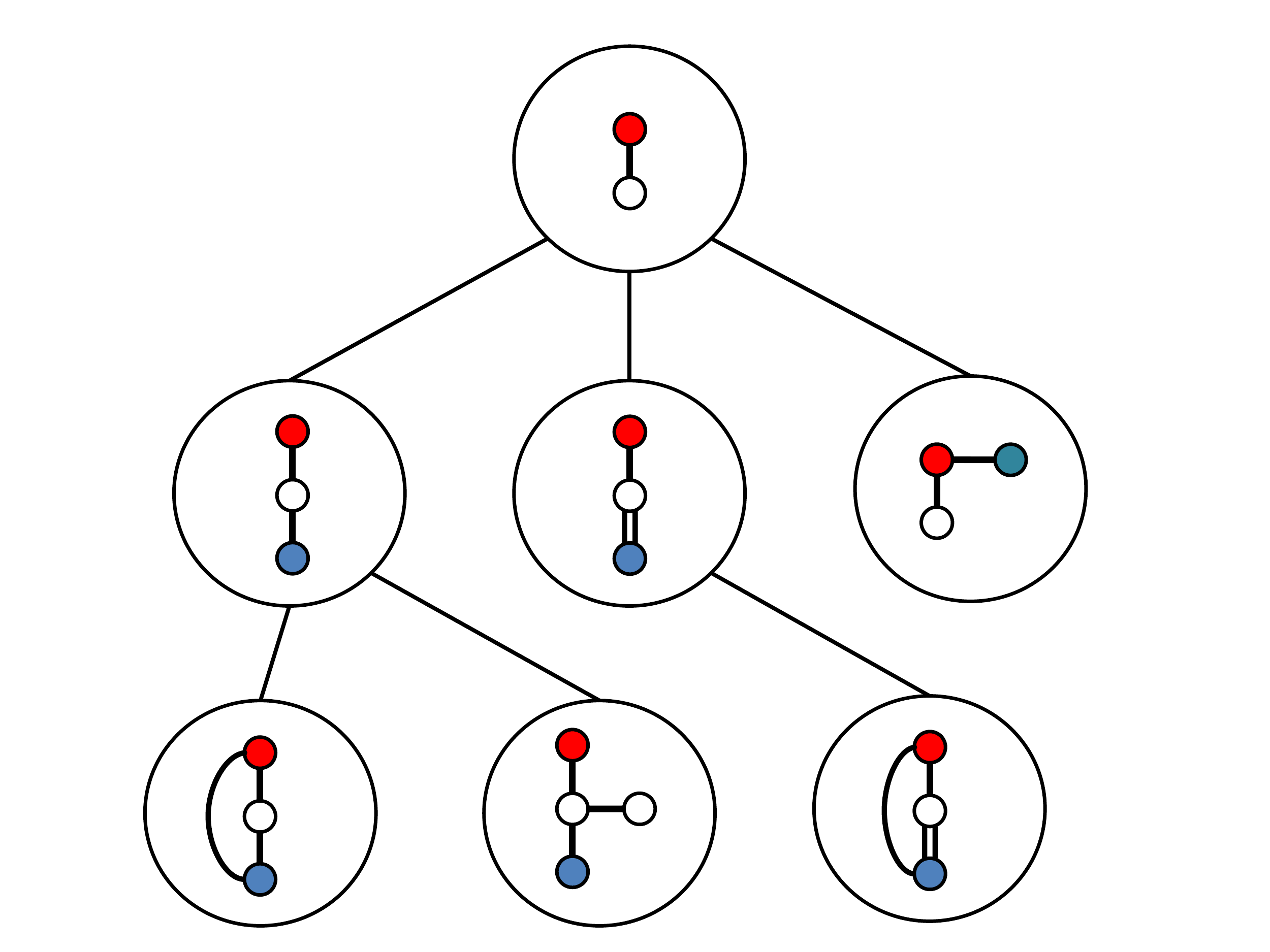}
  \caption{An illustration of a tree structure for graph mining problems. The vertexes are labeled ``red", ``white" or ``blue",
  while the edges are labeled ``single" or ``double" bond.}
  \label{fig:tree_graph}
 \end{center}
\end{figure}

\subsubsection{Problem setup}
We denote the dataset as 
$\{(G_i, y_i)\}_{i \in [n]}$, 
where
$G_i$ is a labeled undirected graph and a response
$y_i$ is defined on $\RR$.
Let $\cJ$
be the set of all possible subgraphs in the database, and denote its size as $J := |\cJ|$.
We denote each of the all subgraphs as $t_1,\cdots,t_J \in \cJ$, and then the occurrence of each pattern is given as the same form
\eq{eq:occurrence_element}.

Note that gSpan is designed for finding subgraphs whose \emph{support} (which is the number of occurrences) is
grater than or equal to minimum support \emph{minsup} and
the maximum number of edges of subgraphs is smaller than or equal to \emph{maxpat}.
In this paper, we only find subgraphs which are highly associated with the response.
To this end, we use the pruning condition \eq{eq:discovery-univariate-condition} during searching subgraphs.
Since the elements of $t_1,\cdots,t_J \in \cJ$ are given in the same form as \eq{eq:occurrence_element},
the problem of searching those subgraphs is inherently the same as the problem of item-set mining discussed in \S\ref{subsec:problem-statement}.
We can apply selective inference to graph mining problems
by using the pruning conditions in
Theorem~\ref{theo:pruning_cond}
by exploiting the anti-monotonicity properties in the tree,
although the number of all subgraphs $J$ is extremely large.

\section{Experiments}
\label{sec:experiments}

\subsection{Experiments on synthetic data (itemset mining)}
\label{subsec:exp-toy}
First,
we compared selective inference ({\tt select}) with naive ({\tt naive}) and data-splitting ({\tt split}) on synthetic data. 
In {\tt naive}, the nominal $p$-values of the $k$ discovered patterns were naively computed without any selection bias correction mechanisms.
In {\tt split}, the data was first divided into two equally sized sets, and one of them was used for pattern discovery, and the other was used for computing $p$-values. 
%
%
Note that
the errors
controlled by these methods
are individual false positive rate 
for each of the discovered patterns 
(although {\tt naive} actually cannot control it),
%
we applied Bonferroni correction within the $k$ discovered patterns, 
i.e.,
we regard a pattern to be positive 
if the Bonferroni-adjusted selective $p$-values 
(obtained by multiplying selective $p$-values by $k$; see \S\ref{subsubsec:propery-selective-p})
is
still
smaller than
the significance level
$\alpha = 0.05$. 
We only considered the problems 
of finding the top $k$ associated patterns 
regardless of the signs of associations 
(the setup described in \S\ref{subsec:pos-neg}).
%
We investigated the results of two scenarios: 
one for finding individual associations 
(indicated as {\tt individual}) 
and
another for finding correlated associations 
by the sequential method 
in \S\ref{subsec:sequential}
(indicated as {\tt sequential}). 

The synthetic data was generated as follows.
In the experiments for comparing false positive rates,
we generated the item-set $T_i$ and the response $y_i$ 
independently at random for each $i \in [n]$. 
The item-set $T_i$ was randomly generated so that it contains
$d (1 - \zeta)$ items on average,
where
$\zeta \in [0, 1]$ 
is an experimental parameter
for representing the sparsity of the data. 
On the other hand,
the response
$y_i$
was randomly generated from a Normal distribution
$N(0, \sigma^2)$. 
In the experiments for comparing true positive rates,
the response $y_i$
was randomly generated from a Normal distribution
$N(\mu(T_i), \sigma^2)$,
where
$\mu(T_i) := 2 \times \one\{\{i_1, i_2, i_3\} \in T_i\}$
in {\tt individual} scenario, 
while
$\mu(T_i) := \frac{1}{2} \times \one\{\{i_1\} \in T_i\} - 2 \times \one\{\{i_2, i_3\} \in T_i\} + 3 \times \one\{\{i_4, i_5, i_6\} \in T_i\}$
in {\tt sequential} scenario.
We investigated the performances by changing various experimental parameters.
We set the baseline parameters as 
$n=100$,
$d=100$, 
$k=5$,
$r=5$,
$\alpha = 0.05$, 
$\sigma=0.5$,
and
$\zeta=0.6$. 

\subsubsection{False positive rates}
\label{subsubsec:exp-toy-FPR}
\figurename~\ref{fig:result-FPR}
shows the false positive rates
when varying the number of transactions
$n \in \{50, 100, \ldots, 250\}$,
the number of items 
$d \in \{50, 100, \ldots, 250\}$.
%
In all cases,
the FW-FPRs of {\tt naive} were far greater than the desired significance level $\alpha = 0.05$,
indicating that the selection bias is harmful.
The FW-FPRs of the other
two
approaches {\tt select} and {\tt split}
were successfully controlled.

\begin{figure}[!hb]
  \centering
  \begin{tabular}{cc}
   \includegraphics[width=0.36\linewidth]{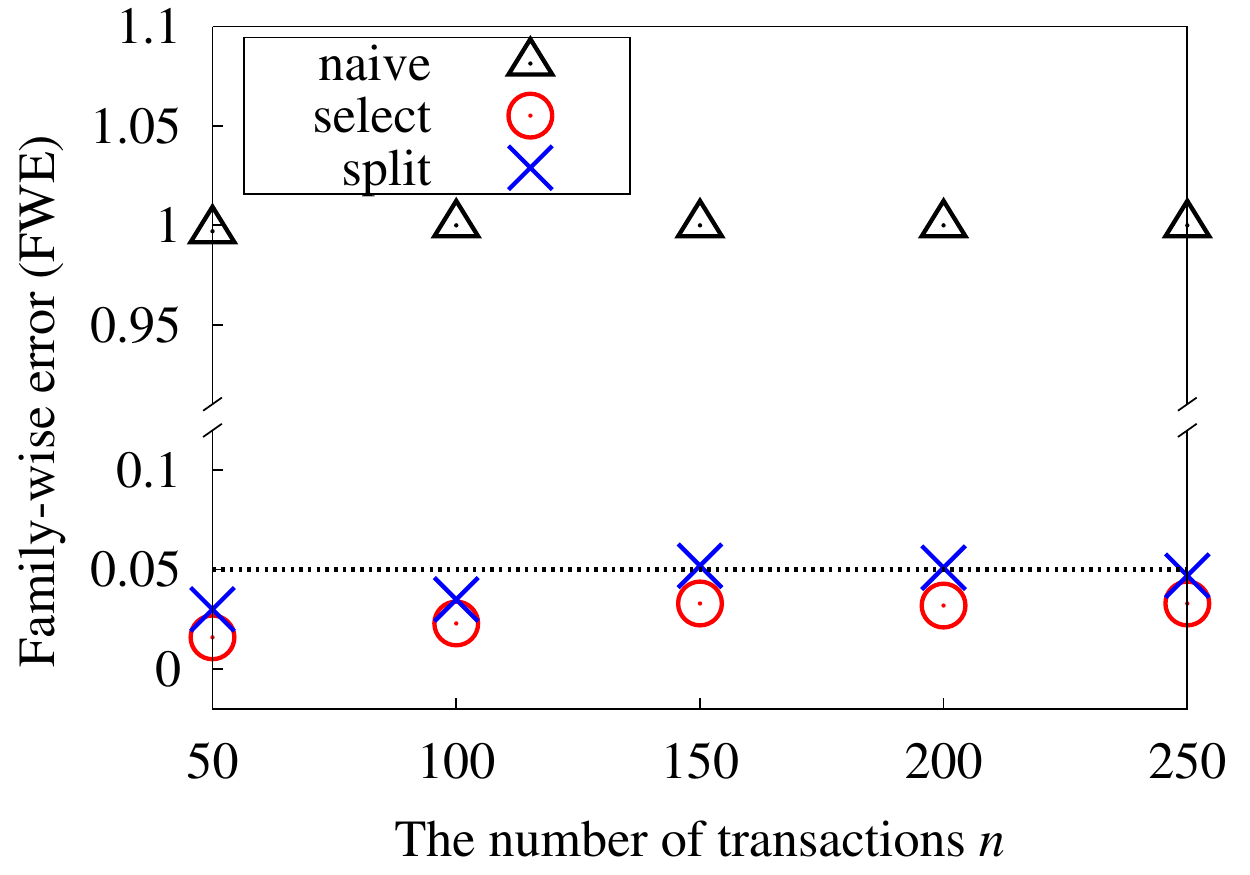} &
   \includegraphics[width=0.36\linewidth]{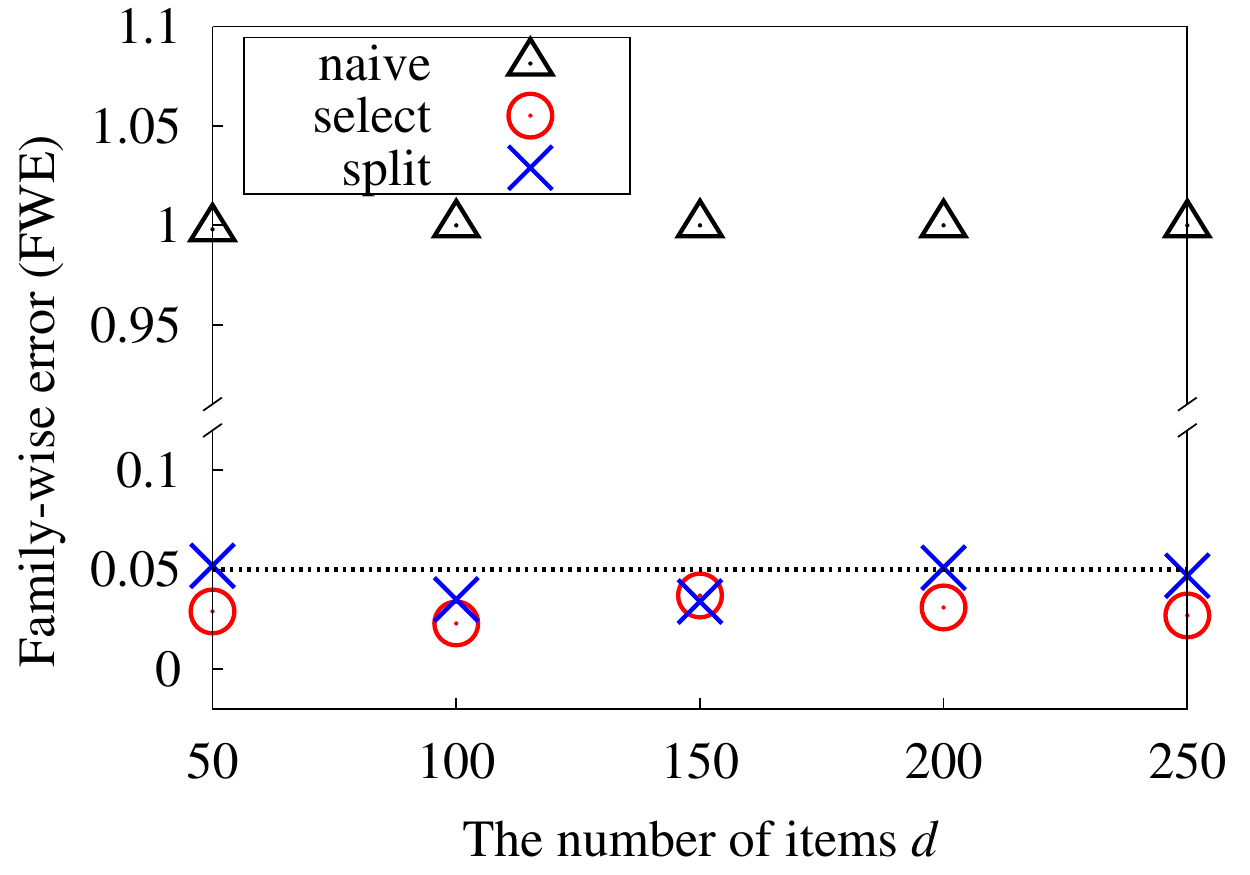} \\
   (a) $n \in \{50,\cdots,250\}$ &
   (b) $d \in \{50,\cdots,250\}$ \\
   \multicolumn{2}{c}{{\tt individual} scenario}
  \end{tabular}
  \begin{tabular}{cc}
   \includegraphics[width=0.36\linewidth]{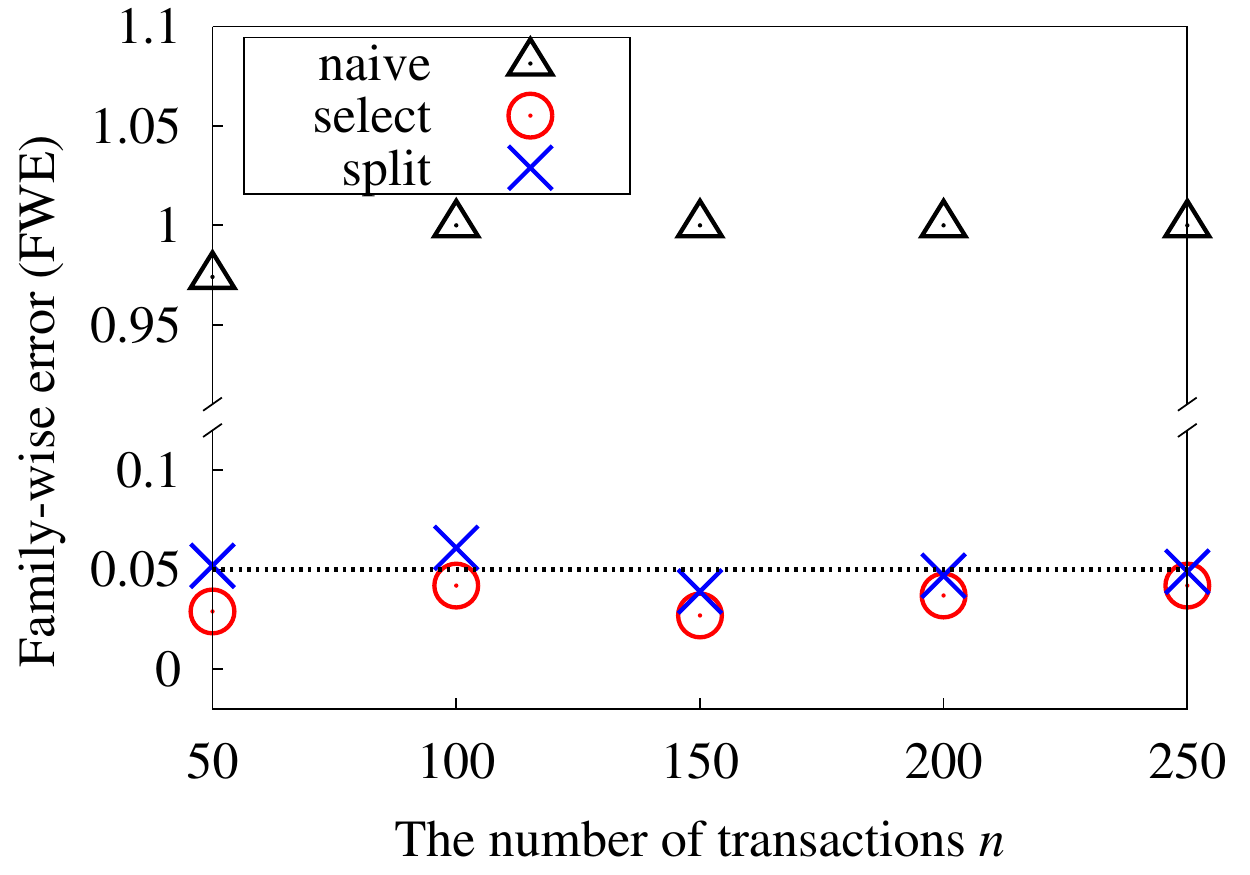} &
   \includegraphics[width=0.36\linewidth]{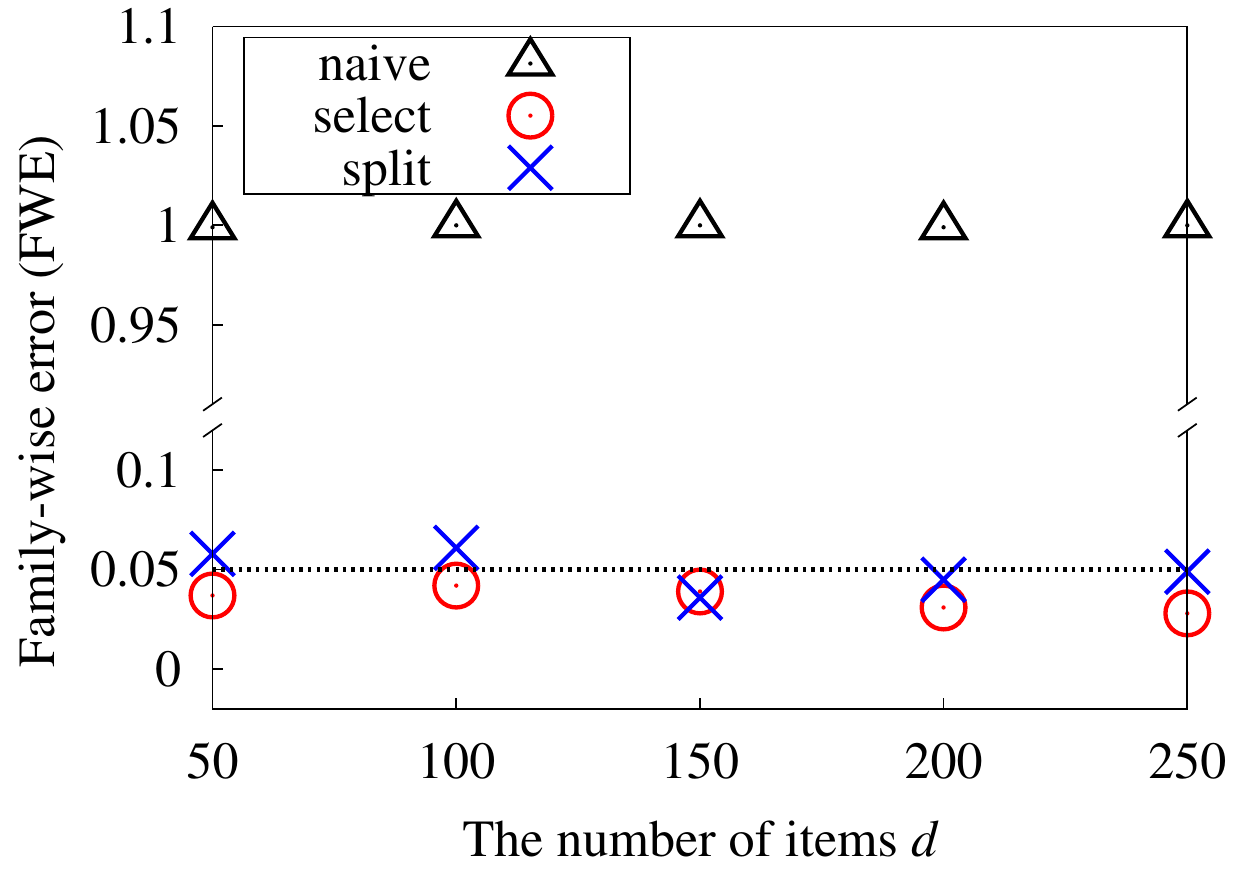} \\
   (c) $n \in \{50,\cdots,250\}$ &
   (d) $d \in \{50,\cdots,250\}$ \\
   \multicolumn{2}{c}{{\tt sequential} scenario}
  \end{tabular}
  \caption{False positive rates.}
  \label{fig:result-FPR}
\end{figure}

\subsubsection{True positive rates}
\label{subsubsec:exp-toy-TPR}
\figurename~\ref{fig:result-TPR}
shows the true positive rates (TPRs)
of {\tt select} and {\tt split}
(we omit {\tt naive} because it cannot control FPRs).
Here,
TPRs
are defined as the probability of discovering truly associated item-sets. 
%
In all experimental setups, 
the TPRs of {\tt select}
were much greater than {\tt split}. 
Note that 
the performances of {\tt split} would be worse than {\tt select} 
both in the discovery and the inference phases. 
The risk of failing to discover truly associated patterns
in {\tt split} 
would be higher
than {\tt select}
because only half of the data would be used in the discovery phase.
Similarly, 
the statistical power of the inference in {\tt split} would be smaller
than {\tt select}
because the sample size is smaller. 
%
%

\begin{figure}[!ht]
  \centering
 \vspace*{-1mm}
 \begin{tabular}{cc}
   \includegraphics[width=0.39\linewidth]{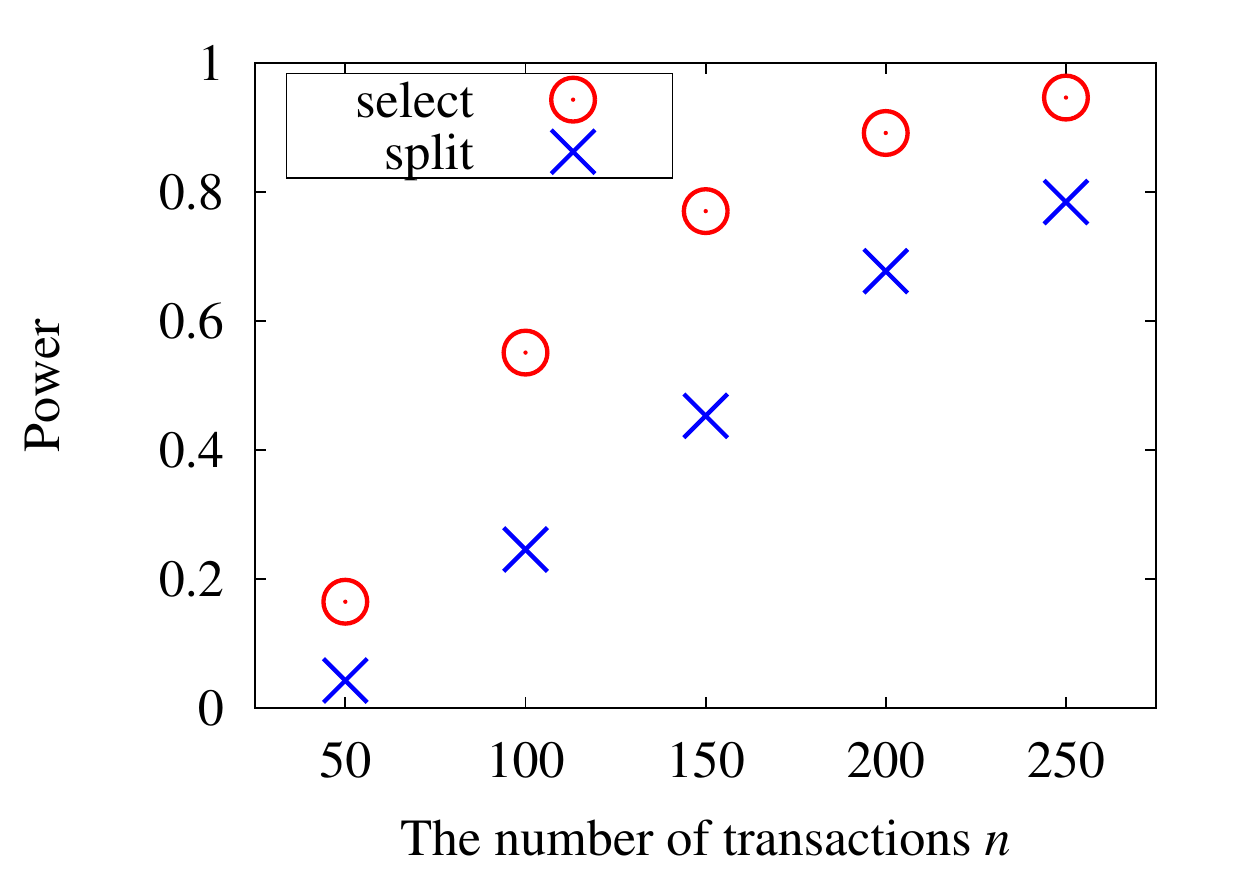} &
   \includegraphics[width=0.39\linewidth]{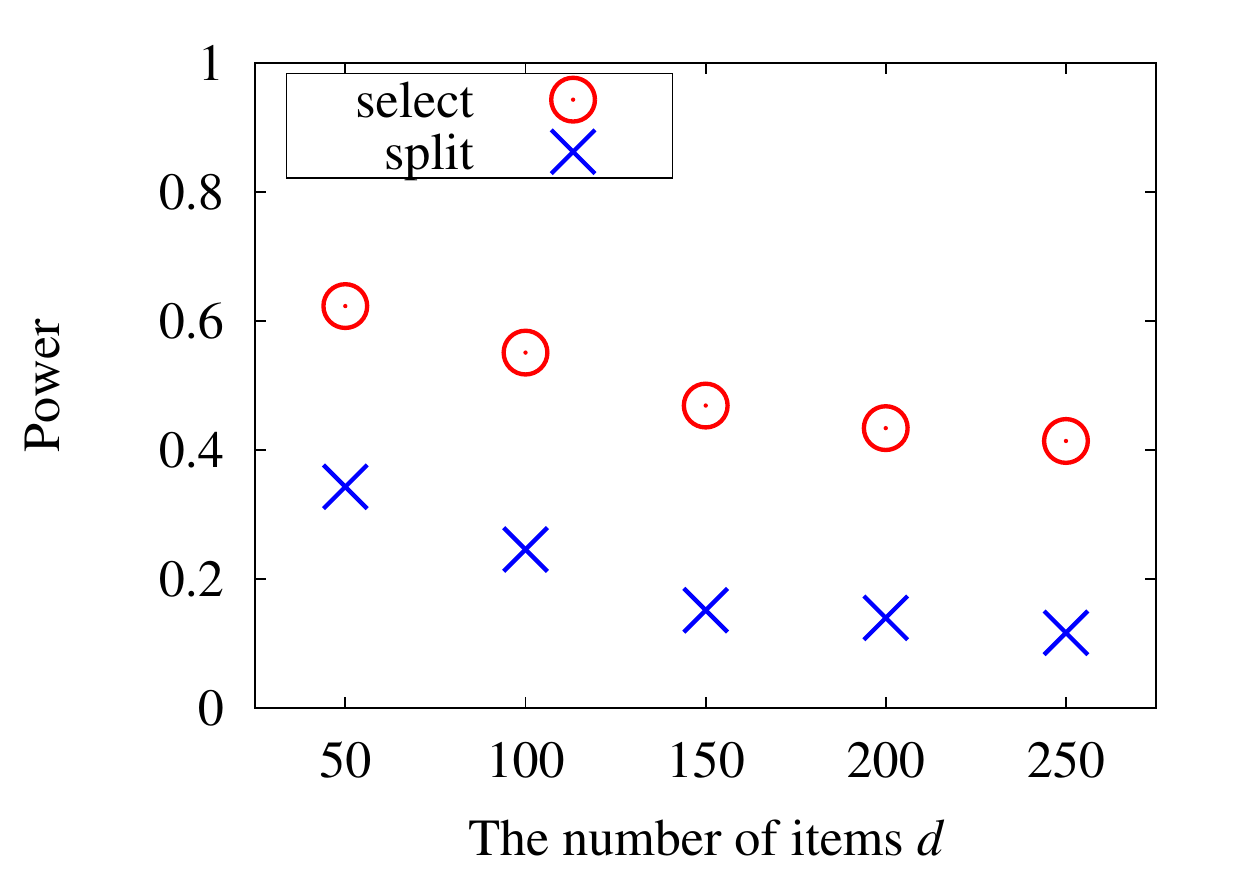} \\
   (a) $n \in \{50,\cdots,250\}$ &
   (b) $d \in \{50,\cdots,250\}$ \\
   \multicolumn{2}{c}{{\tt individual} scenario}
 \end{tabular}
  \begin{tabular}{cc}
   \includegraphics[width=0.39\linewidth]{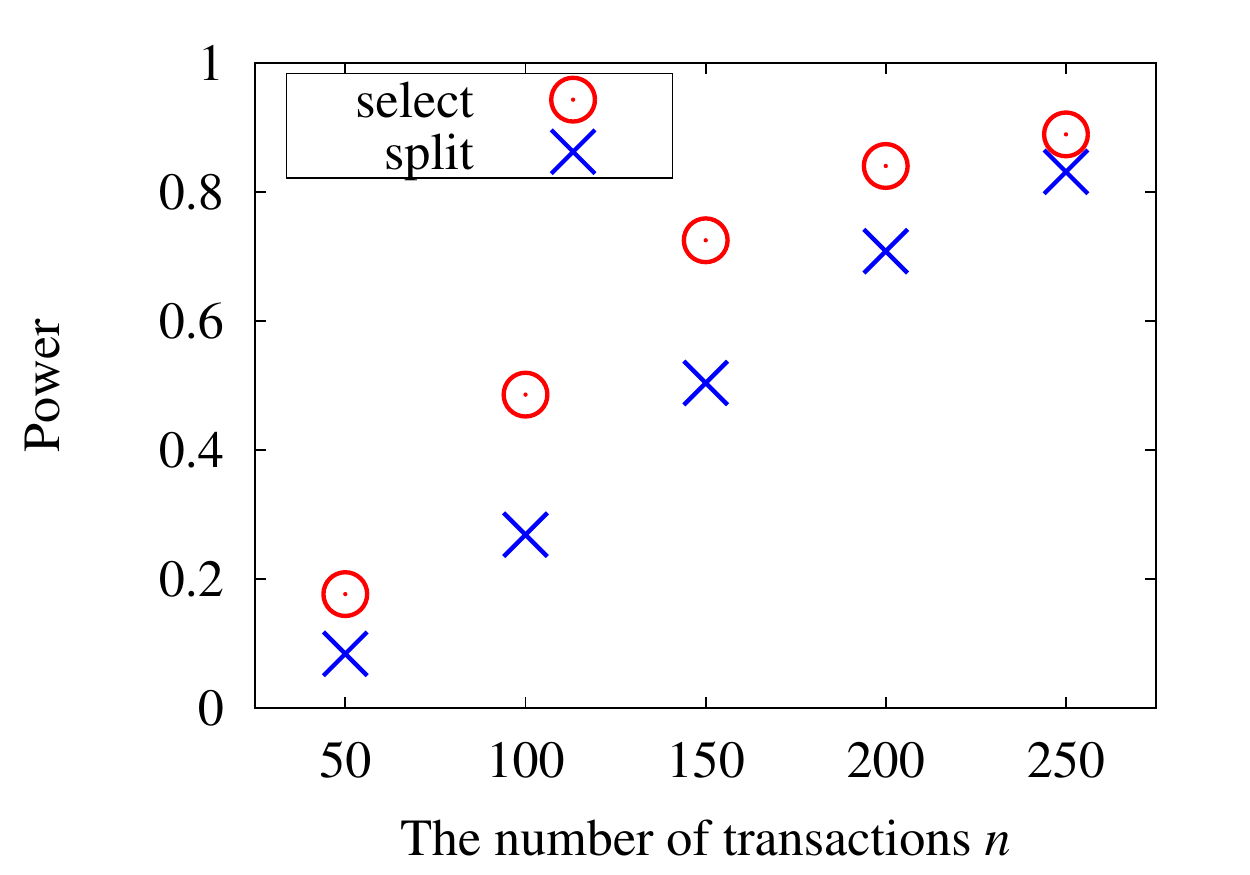} &
   \includegraphics[width=0.39\linewidth]{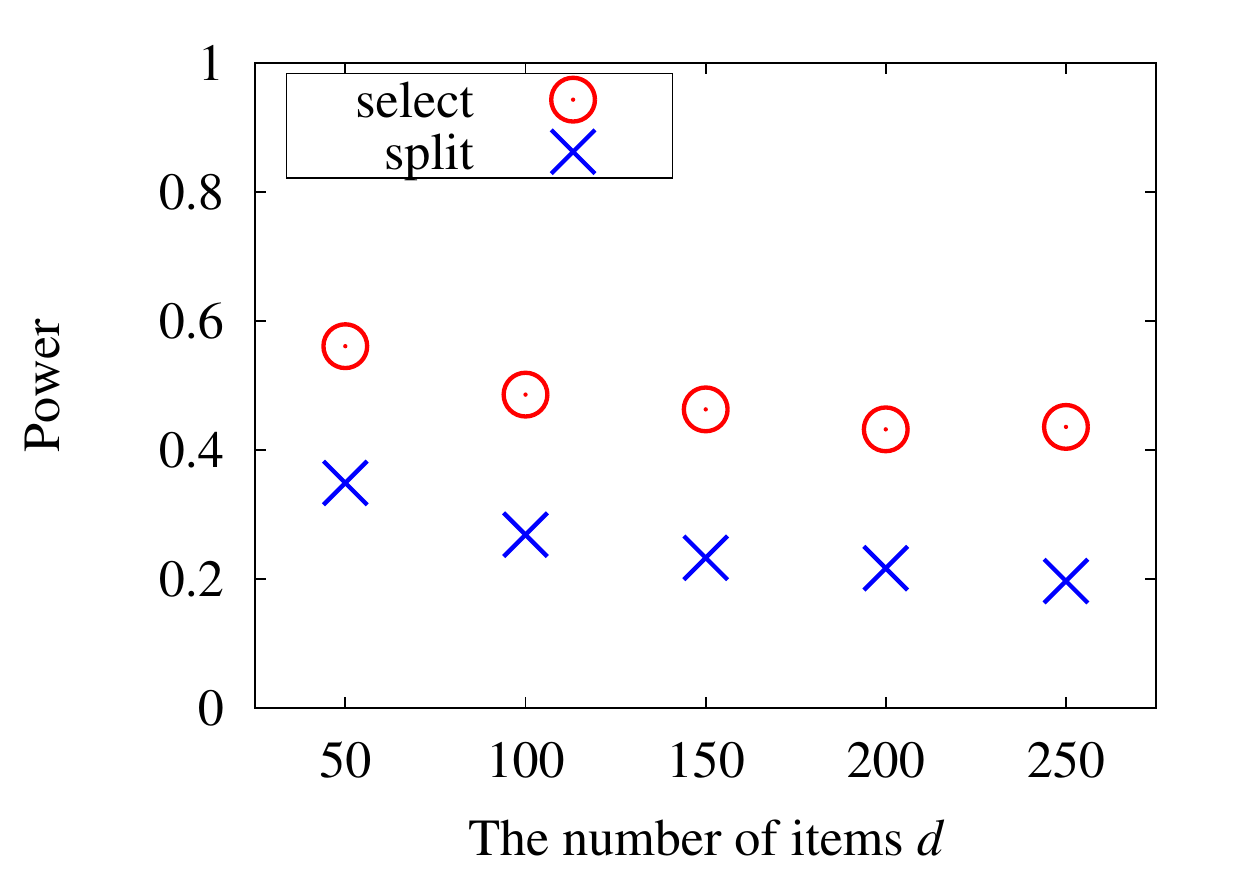} \\
   (c) $n \in \{50,\cdots,250\}$ &
   (d) $d \in \{50,\cdots,250\}$ \\
   \multicolumn{2}{c}{{\tt sequential} scenario}
  \end{tabular}
  \caption{True positive rates.}
  \label{fig:result-TPR}
\end{figure}

\subsubsection{Computational efficiency}
\label{subsubsec:exp-toy-efficiency}
Table~\ref{tab:result_time}
shows the computation times in seconds for the selective inference approach 
with and without 
the computational tricks described in
\S\ref{sec:SI-for-PM} 
for various values of 
the number of transactions
$n \in \{100, \ldots, 10000\}$,
the number of items 
$d \in \{100, \ldots, 10000\}$,
and
the sparsity rates 
$\zeta \in \{0.8, 0.9\}$
(we terminated the search if the time exceeds 1 day). 
It can be observed from the table that,
if we use the computational trick,
the selective inferences can be conducted with reasonable computational costs
except for $d \ge 5000$ and $\zeta = 0.8$ cases in {\tt sequential} scenario.
When the computational trick was not used, 
the cost was extremely large.
Especially when the number of items $d$ is larger than 100,
we could not complete the search within 1 day. 
From the results, 
we conclude that
computational trick described in \S\ref{sec:SI-for-PM} is indispensable 
for selective inferences in pattern mining problems. 
\begin{table}[!ht]
 \centering
 \scriptsize
 \caption{Computation times [sec]}
 \label{tab:result_time}
 \begin{tabular}{cc}
 {\tt individual} scenario & {\tt sequential} scenario \\
 \begin{tabular}{r||l|l||l|l} \hline
 & \multicolumn{2}{c||}{\tiny with computational trick} & \multicolumn{2}{c}{\tiny without computational trick} \\ \cline{2-5}
$n$	&	\multicolumn{1}{c|}{$\zeta=0.8$}	&	\multicolumn{1}{c||}{$\zeta=0.9$}	&	\multicolumn{1}{c|}{$\zeta=0.8$}	&	\multicolumn{1}{c}{$\zeta=0.9$} \\ \hline
100	&	$4.68\times10^{-2}$	&	$1.80\times10^{-2}$	&	$1.37\times10^{2}$	&	$1.31\times10^{2}$	\\
500	&	$1.74\times10^{-1}$	&	$9.07\times10^{-2}$	&	$1.80\times10^{2}$	&	$1.36\times10^{2}$	\\
1000	&	$3.38\times10^{-1}$	&	$1.54\times10^{-1}$	&	$2.65\times10^{2}$	&	$1.41\times10^{2}$	\\
5000	&	$2.33\times10^{0}$	&	$6.61\times10^{-1}$	&	$1.05\times10^{3}$	&	$2.57\times10^{2}$	\\
10000	&	$5.04\times10^{0}$	&	$1.55\times10^{0}$	&	$2.06\times10^{3}$	&	$5.12\times10^{2}$	\\ \hline
 \end{tabular}
 &
 \begin{tabular}{l|l||l|l} \hline
 \multicolumn{2}{c||}{\tiny with computational trick} & \multicolumn{2}{c}{\tiny without computational trick} \\ \hline
\multicolumn{1}{c|}{$\zeta=0.8$}	&	\multicolumn{1}{c||}{$\zeta=0.9$}	&	\multicolumn{1}{c|}{$\zeta=0.8$}	&	\multicolumn{1}{c}{$\zeta=0.9$} \\ \hline
$2.33\times10^{-1}$	&	$5.85\times10^{-2}$	&	$8.83\times10^{2}$	&	$8.28\times10^{2}$	\\
$1.01\times10^{0}$	&	$3.74\times10^{-1}$	&	$1.33\times10^{3}$	&	$8.60\times10^{2}$	\\
$3.18\times10^{0}$	&	$7.27\times10^{-1}$	&	$2.15\times10^{3}$	&	$9.07\times10^{2}$	\\
$6.20\times10^{1}$	&	$3.48\times10^{0}$	&	$1.00\times10^{4}$	&	$2.05\times10^{3}$	\\
$1.24\times10^{2}$	&	$9.00\times10^{0}$	&	$1.98\times10^{4}$	&	$4.63\times10^{3}$	\\ \hline
 \end{tabular}
 \\
 \begin{tabular}{r||l|l||l|l} \hline
$d$	&	\multicolumn{1}{c|}{$\zeta=0.8$}	&	\multicolumn{1}{c||}{$\zeta=0.9$}	&	\multicolumn{1}{c|}{$\zeta=0.8$}	&	\multicolumn{1}{c}{$\zeta=0.9$} \\ \hline
100	&	$4.40\times10^{-2}$	&	$1.77\times10^{-2}$	&	$1.47\times10^{2}$	&	$1.31\times10^{2}$	\\
500	&	$5.06\times10^{-1}$	&	$1.64\times10^{-1}$	&	$\ge$ 1 day	&	$\ge$ 1 day	\\
1000	&	$1.23\times10^{0}$	&	$3.74\times10^{-1}$	&	$\ge$ 1 day	&	$\ge$ 1 day	\\
5000	&	$1.53\times10^{1}$	&	$2.88\times10^{0}$	&	$\ge$ 1 day	&	$\ge$ 1 day	\\
10000	&	$3.70\times10^{1}$	&	$6.16\times10^{0}$	&	$\ge$ 1 day	&	$\ge$ 1 day	\\ \hline
 \end{tabular}
 &
 \begin{tabular}{l|l||l|l} \hline
 \multicolumn{1}{c|}{$\zeta=0.8$}	&	\multicolumn{1}{c||}{$\zeta=0.9$}	&	\multicolumn{1}{c|}{$\zeta=0.8$}	&	\multicolumn{1}{c}{$\zeta=0.9$} \\ \hline
$2.41\times10^{-1}$	&	$6.02\times10^{-2}$	&	$8.86\times10^{2}$	&	$8.20\times10^{2}$	\\
$3.52\times10^{1}$	&	$9.83\times10^{0}$	&	$\ge$ 1 day	&	$\ge$ 1 day	\\
$3.01\times10^{2}$	&	$1.66\times10^{2}$	&	$\ge$ 1 day	&	$\ge$ 1 day	\\
$\ge$ 1 day	&	$1.92\times10^{3}$	&	$\ge$ 1 day	&	$\ge$ 1 day	\\
$\ge$ 1 day	&	$5.98\times10^{4}$	&	$\ge$ 1 day	&	$\ge$ 1 day	\\ \hline
 \end{tabular}
 \end{tabular}
\end{table}

\newpage
\subsection{Application to HIV drug resistance data (itemset mining)}
\label{subsec:exp-itemset}
We applied the selective inference approach to HIV-1 sequence data obtained from Stanford HIV Drug Resistance Database \cite{rhee2003human}. 
The goal here is to find statistically significant high-order interactions of multiple mutations (up to $r=5$ order interactions) that are highly associated with drug resistances.
Same datasets were also studied in \cite{saigo2007mining}. 
We discovered $k=30$ patterns, and evaluated the statistical significances of these patterns by selective inference. 
Table~\ref{tab:HIV-result}
shows the numbers of 1st, 2nd, 3rd and 4th order interactions
that were statistically significant in the sense that the Bonferroni adjusted selective $p$-values is smaller than $\alpha = 0.05$ 
(there were no statistically significant 5th order interactions).
\figurename~\ref{fig:HIV-pval-result}
shows
the list of Bonferroni-adjusted selective $p$-values in increasing order 
on {\tt idv} and {\tt d4t} datasets 
in {\tt individual} and {\tt sequential} scenario, respectively. 
These results indicate that
selective inference approach could successfully identify statistically significant
high-order interactions of multiple mutations. 

\begin{table}[!ht]
 \centering
 \scriptsize
 \caption{The numbers of significant high-order interactions of multiple mutations in HIV datasets.}
 \label{tab:HIV-result}
 \begin{tabular}{r|c|c|c|c|c||c|c|c|c|c} \hline
 & \multicolumn{5}{c||}{{\tt individual} scenario} & \multicolumn{5}{c}{{\tt sequential} scenario} \\ \cline{2-11}
Data	&	\!\!$1^{\rm st}$\!\!	&	\!\!$2^{\rm nd}$\!\!	&	\!\!$3^{\rm rd}$\!\!	&	\!\!$4^{\rm th}$\!\!	&	\!Time[s]\!	&	\!\!$1^{\rm st}$\!\!	&	\!\!$2^{\rm nd}$\!\!	&	\!\!$3^{\rm rd}$\!\!	&	\!\!$4^{\rm th}$\!\!	&	\!Time[s]\!	\\ \hline
\multicolumn{11}{c}{NNRTI ($d=371$)} \\ \hline 
dlv($n=732$)	&	1	&		&		&		&	.495	&	2	&		&		&		&	18.0	\\
efv($n=734$)	&		&		&		&		&	.732	&	5	&		&		&		&	13.7	\\
nvp($n=746$)	&	4	&	1	&		&		&	.774	&	8	&		&		&		&	17.4	\\ \hline
\multicolumn{11}{c}{NRTI ($d=348$)} \\ \hline 
3tc($n=633$)	&	1	&	2	&		&		&	.257	&	4	&		&		&		&	15.1	\\
abc($n=628$)	&	5	&	13	&	7	&	2	&	.238	&	9	&		&		&		&	11.7	\\
azt($n=630$)	&	2	&	5	&	3	&	1	&	.231	&	5	&		&		&		&	17.5	\\
d4t($n=630$)	&	4	&	11	&	6	&	1	&	.215	&	7	&	1	&	3	&		&	13.7	\\
ddi($n=632$)	&	2	&	1	&		&		&	.234	&	6	&		&		&		&	12.1	\\
tdf($n=353$)	&		&		&		&		&	.230	&	3	&	1	&		&		&	26.4	\\ \hline
\multicolumn{11}{c}{PI ($d=225$)} \\ \hline 
apv($n=768$)	&	3	&	6	&	1	&		&	.188	&	9	&		&		&		&	~6.5	\\
atv($n=329$)	&	1	&	3	&	2	&		&	.150	&	3	&	1	&		&		&	~5.0	\\
idv($n=827$)	&	1	&	6	&	3	&		&	.437	&	9	&		&		&		&	~6.2	\\
lpv($n=517$)	&	4	&	4	&	1	&		&	.275	&	11	&		&		&		&	~6.1	\\
nfv($n=844$)	&	5	&	7	&	1	&		&	.455	&	15	&		&		&		&	~5.8	\\
rtv($n=795$)	&	5	&	7	&	2	&		&	.183	&	10	&	1	&		&		&	~5.6	\\
sqv($n=826$)	&	1	&	3	&	2	&		&	.623	&	7	&	1	&		&		&	~7.8	\\ \hline
 \end{tabular}
\end{table}

\begin{figure}[!ht]
 \begin{center}
  \includegraphics[width=0.65\linewidth]{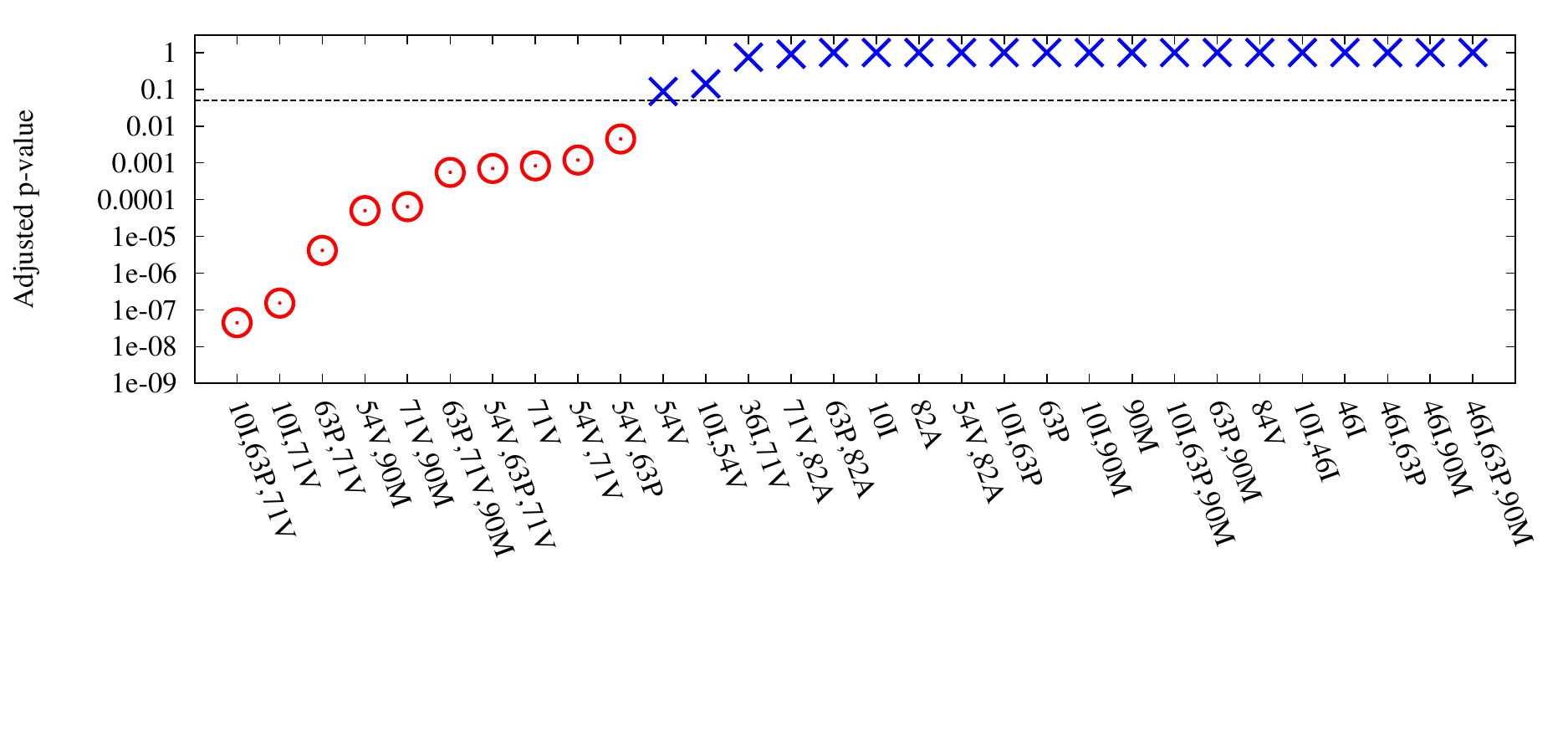}\\
  \vspace*{-7.5mm}
  (a) {\tt  {\tt idv} dataset ({\tt individual} scenario)}\\
  \includegraphics[width=0.65\linewidth]{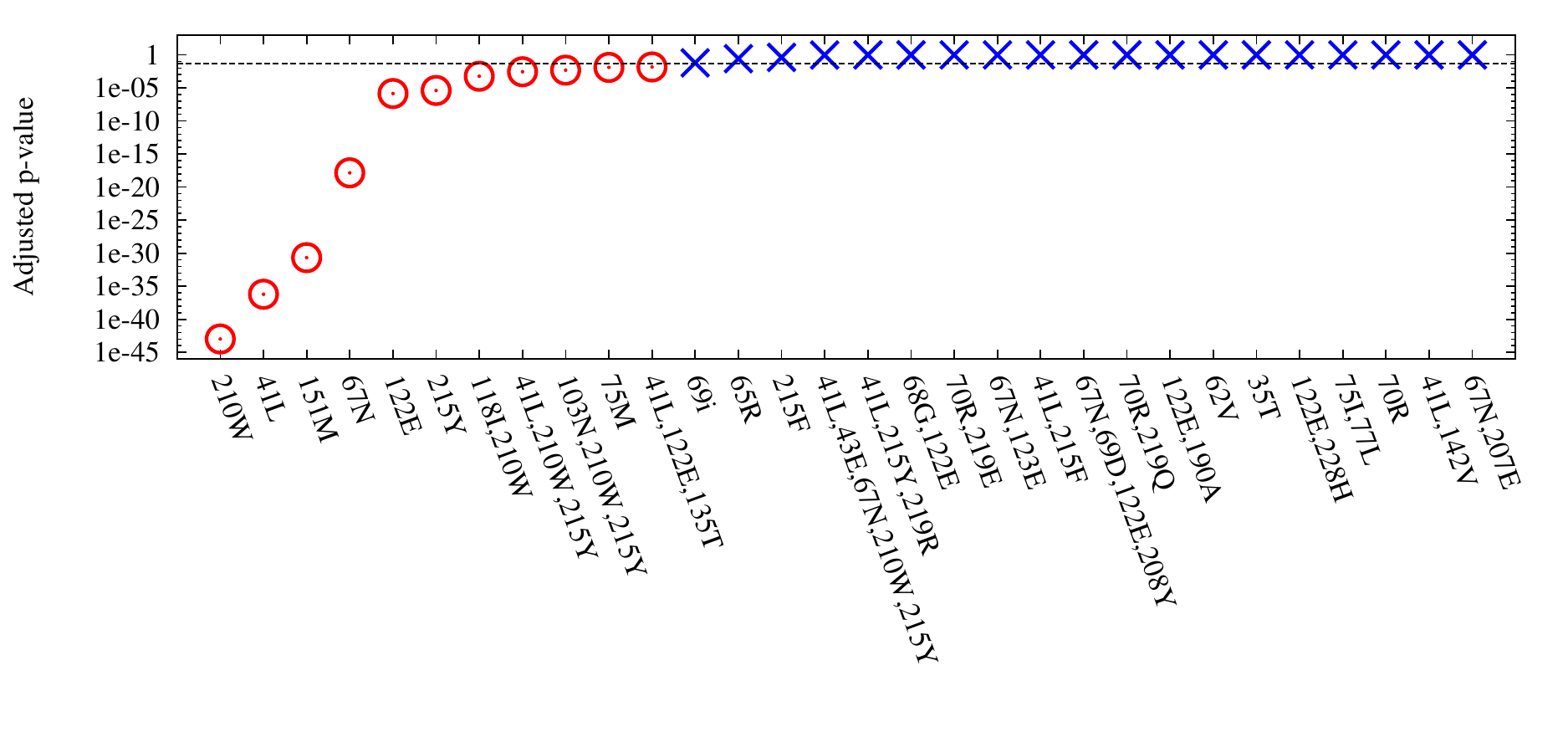}\\
  \vspace*{-2.5mm}
  (b) {\tt  {\tt d4t} dataset ({\tt sequential} scenario)}
 \end{center}
  \caption{
 The list of Bonferroni-adjusted selective $p$-values
 of $k=30$ discovered high-order interactions of multiple mutations 
 on two HIV datasets.
 }
  \label{fig:HIV-pval-result}
\end{figure}

\newpage
\subsection{Experiments on graph mining with chemical data}
\label{subsec:exp-graph}
Here we used {\tt Karthikeyan} dataset where the response is the
melting point of each of the $n = 4173$ chemical compounds
(this data is available at \url{http://cheminformatics.org/datasets/}).
We considered the case with ${\tt maxpat} = \infty$ which indicates the maximum number of edges of subgraphs we wanted to find.
We discovered $k$ = 50 subgraphs which are individually associated with the melting point, 
and evaluated the statistical significances of
those subgraphs by selective inference.
Table~\ref{tab:Graph-result}
shows the numbers of subgraphs
that were statistically significant in the sense that the Bonferroni adjusted selective $p$-values are smaller than $\alpha = 0.05$,
where the identified subgraphs contain up to 7 edges
(there were no statistically significant subgraphs that have more than 7 edges).
\figurename~\ref{fig:Graph-pval-result}
shows
the list of 20 subgraphs and Bonferroni-adjusted selective $p$-values in increasing order.
These results indicate that
selective inference approach could identify statistically significant subgraphs at reasonable computational costs. 

\begin{table}[!ht]
 \centering
 \caption{The numbers of significant subgraphs in {\tt Karthikeyan} dataset.}
 \label{tab:Graph-result}
 \begin{tabular}{c|c|c|c|c|c|c|c} \hline
 $1^{\rm st}$	&	$2^{\rm nd}$	&	$3^{\rm rd}$	&	$4^{\rm th}$	&
 $5^{\rm th}$	&	$6^{\rm th}$	&	$7^{\rm th}$	&	Time[s] \\ \hline
 3 &	5 &	7 &	7 &	8 &	6 &	1 & 5.4 \\ \hline
 \end{tabular}
\end{table}

\begin{figure}[!ht]
 \centering
 \begin{tabular}{|c|c|c|c|} \hline
  \includegraphics[width=0.145\linewidth]{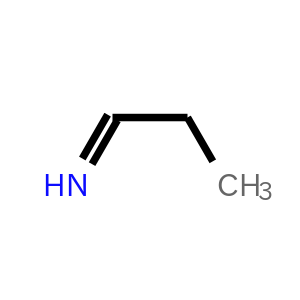} &
  \includegraphics[width=0.145\linewidth]{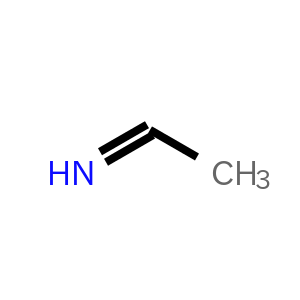} &
  \includegraphics[width=0.145\linewidth]{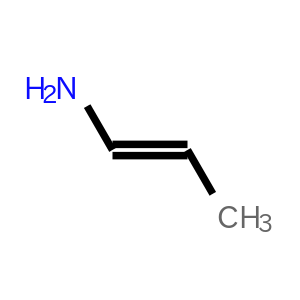} &
  \includegraphics[width=0.115\linewidth]{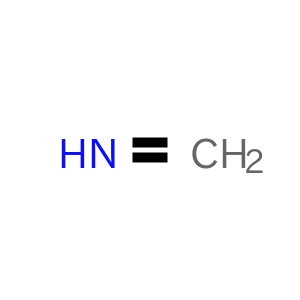} \\
  $4.54 \times 10^{-16}$ & $1.76 \times 10^{-13}$ & $1.70 \times 10^{-12}$ & $2.03 \times 10^{-12}$ \\ \hline
  \includegraphics[width=0.155\linewidth]{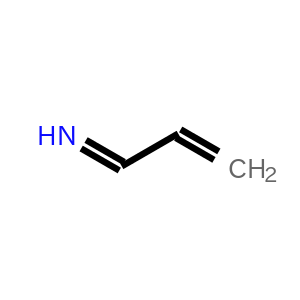} &
  \includegraphics[width=0.145\linewidth]{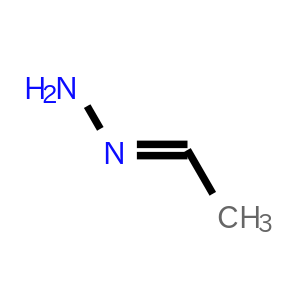} &
  \includegraphics[width=0.145\linewidth]{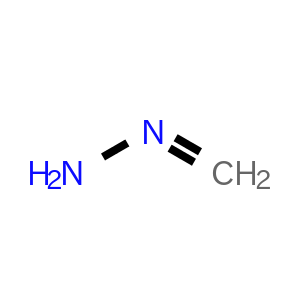} &
  \includegraphics[width=0.115\linewidth]{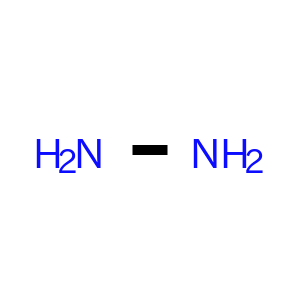} \\
  $6.47 \times 10^{-12}$ & $1.04 \times 10^{-11}$ & $1.04 \times 10^{-11}$ & $1.35 \times 10^{-11}$ \\ \hline
  \includegraphics[width=0.165\linewidth]{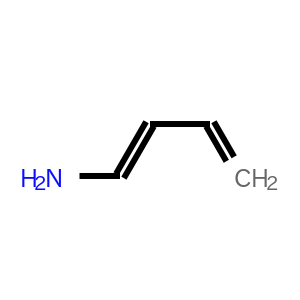} &
  \includegraphics[width=0.18\linewidth]{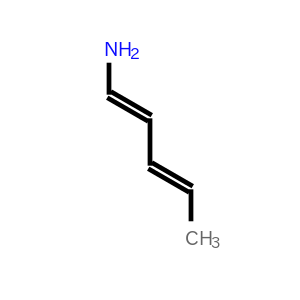} &
  \includegraphics[width=0.145\linewidth]{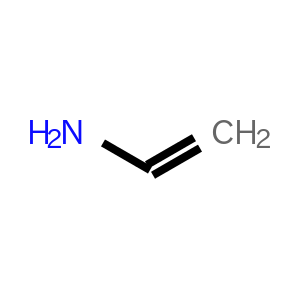} &
  \includegraphics[width=0.145\linewidth]{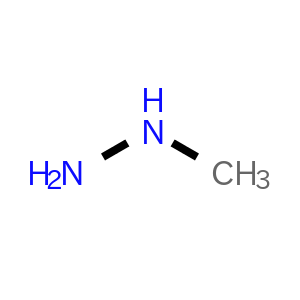} \\
  $3.36 \times 10^{-11}$ & $3.89 \times 10^{-11}$ & $5.85 \times 10^{-11}$ & $2.17 \times 10^{-10}$ \\ \hline
  \includegraphics[width=0.175\linewidth]{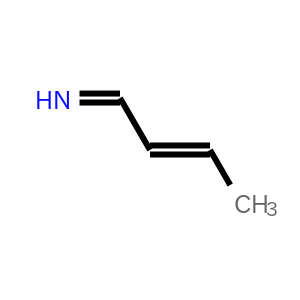} &
  \includegraphics[width=0.17\linewidth]{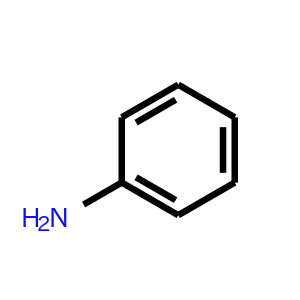} &
  \includegraphics[width=0.185\linewidth]{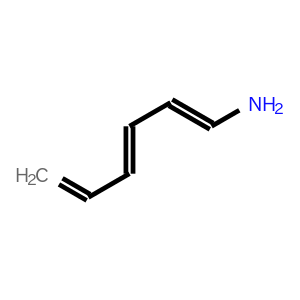} &
  \includegraphics[width=0.175\linewidth]{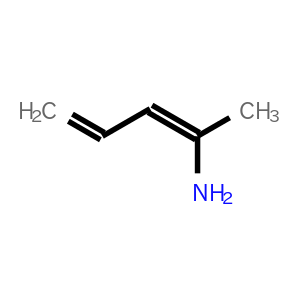} \\
  $3.51 \times 10^{-10}$ & $6.59 \times 10^{-10}$ & $1.55 \times 10^{-9}$ & $1.03 \times 10^{-8}$ \\ \hline
  \includegraphics[width=0.185\linewidth]{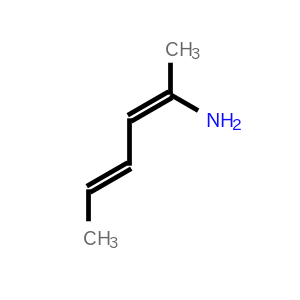} &
  \includegraphics[width=0.165\linewidth]{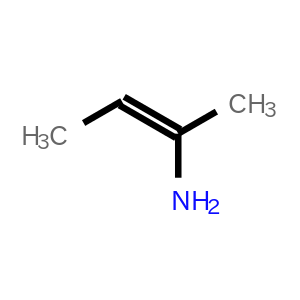} &
  \includegraphics[width=0.185\linewidth]{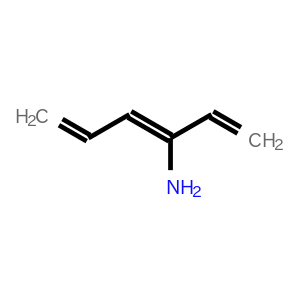} &
  \includegraphics[width=0.175\linewidth]{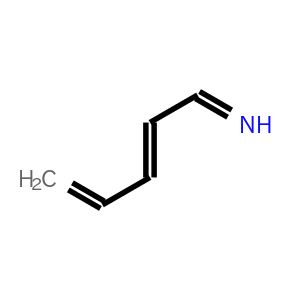} \\
  $1.39 \times 10^{-8}$ & $1.49 \times 10^{-8}$ & $2.01 \times 10^{-8}$ & $3.85 \times 10^{-8}$ \\ \hline
 \end{tabular}
 \caption{The list of 20 subgraphs and Bonferroni-adjusted selective $p$-values.
 The label ``H(hydrogen)" is omitted when the lebel of vertex is only ``H".}
 \label{fig:Graph-pval-result}
\end{figure}

\section{Conclusion}
\label{sec:conclusion}
In this paper 
we extended selective inference framework
to
predictive pattern mining problems
by introducing a novel computational trick
for computing selective sampling distribution
for a class of mining algorithms.
We demonstrate that
selective inference approach is useful
for
finding statistically sound patterns
from databases
because it allows us to address selection bias issue. 

\bibliographystyle{IEEEtran}
\bibliography{paper}

\begin{thebibliography}{10}
\providecommand{\url}[1]{#1}
\csname url@samestyle\endcsname
\providecommand{\newblock}{\relax}
\providecommand{\bibinfo}[2]{#2}
\providecommand{\BIBentrySTDinterwordspacing}{\spaceskip=0pt\relax}
\providecommand{\BIBentryALTinterwordstretchfactor}{4}
\providecommand{\BIBentryALTinterwordspacing}{\spaceskip=\fontdimen2\font plus
\BIBentryALTinterwordstretchfactor\fontdimen3\font minus
  \fontdimen4\font\relax}
\providecommand{\BIBforeignlanguage}[2]{{%
\expandafter\ifx\csname l@#1\endcsname\relax
\typeout{** WARNING: IEEEtran.bst: No hyphenation pattern has been}%
\typeout{** loaded for the language `#1'. Using the pattern for}%
\typeout{** the default language instead.}%
\else
\language=\csname l@#1\endcsname
\fi
#2}}
\providecommand{\BIBdecl}{\relax}
\BIBdecl

\bibitem{hamalainen2014statistically}
W.~H{\"a}m{\"a}l{\"a}inen and G.~Webb, ``Statistically sound pattern
  discovery,'' in \emph{Tutorial of the 20th ACM SIGKDD international
  conference on Knowledge discovery and data mining}.\hskip 1em plus 0.5em
  minus 0.4em\relax ACM, 2014.

\bibitem{webb2007discovering}
G.~I. Webb, ``Discovering significant patterns,'' \emph{Machine Learning},
  vol.~68, no.~1, pp. 1--33, 2007.

\bibitem{fan2008direct}
W.~Fan, K.~Zhang, H.~Cheng, J.~Gao, X.~Yan, J.~Han, P.~Yu, and O.~Verscheure,
  ``Direct mining of discriminative and essential frequent patterns via
  model-based search tree,'' in \emph{Proceedings of the 14th ACM SIGKDD
  international conference on Knowledge discovery and data mining}.\hskip 1em
  plus 0.5em minus 0.4em\relax ACM, 2008, pp. 230--238.

\bibitem{novak2009supervised}
P.~K. Novak, N.~Lavra{\v{c}}, and G.~I. Webb, ``Supervised descriptive rule
  discovery: A unifying survey of contrast set, emerging pattern and subgroup
  mining,'' \emph{The Journal of Machine Learning Research}, vol.~10, pp.
  377--403, 2009.

\bibitem{zimmermann2014supervised}
A.~Zimmermann and S.~Nijssen, ``Supervised pattern mining and applications to
  classification,'' in \emph{Frequent Pattern Mining}.\hskip 1em plus 0.5em
  minus 0.4em\relax Springer, 2014, pp. 425--442.

\bibitem{heckman1979sample}
J.~J. Heckman, ``Sample selection bias as a specification error,''
  \emph{Econometrica: Journal of the econometric society}, pp. 153--161, 1979.

\bibitem{taylor2015statistical}
J.~Taylor and R.~J. Tibshirani, ``Statistical learning and selective
  inference,'' \emph{Proceedings of the National Academy of Sciences}, vol.
  112, no.~25, pp. 7629--7634, 2015.

\bibitem{berk2013valid}
R.~Berk, L.~Brown, A.~Buja, K.~Zhang, L.~Zhao \emph{et~al.}, ``Valid
  post-selection inference,'' \emph{The Annals of Statistics}, vol.~41, no.~2,
  pp. 802--837, 2013.

\bibitem{lee2013exact}
J.~D. Lee, D.~L. Sun, Y.~Sun, and J.~E. Taylor, ``Exact post-selection
  inference with the lasso,'' \emph{arXiv preprint arXiv:1311.6238}, 2013.

\bibitem{agrawal1993mining}
R.~Agrawal, T.~Imieli{\'n}ski, and A.~Swami, ``Mining association rules between
  sets of items in large databases,'' in \emph{ACM SIGMOD Record}, vol.~22,
  no.~2.\hskip 1em plus 0.5em minus 0.4em\relax ACM, 1993, pp. 207--216.

\bibitem{bay2001detecting}
S.~D. Bay and M.~J. Pazzani, ``Detecting group differences: Mining contrast
  sets,'' \emph{Data Mining and Knowledge Discovery}, vol.~5, no.~3, pp.
  213--246, 2001.

\bibitem{yan2008mining}
X.~Yan, H.~Cheng, J.~Han, and P.~S. Yu, ``Mining significant graph patterns by
  leap search,'' in \emph{Proceedings of the 2008 ACM SIGMOD international
  conference on Management of data}.\hskip 1em plus 0.5em minus 0.4em\relax
  ACM, 2008, pp. 433--444.

\bibitem{hamalainen2010statapriori}
W.~H{\"a}m{\"a}l{\"a}inen, ``Statapriori: an efficient algorithm for searching
  statistically significant association rules,'' \emph{Knowledge and
  information systems}, vol.~23, no.~3, pp. 373--399, 2010.

\bibitem{arora2014mining}
A.~Arora, M.~Sachan, and A.~Bhattacharya, ``Mining statistically significant
  connected subgraphs in vertex labeled graphs,'' in \emph{Proceedings of the
  2014 ACM SIGMOD international conference on Management of data}.\hskip 1em
  plus 0.5em minus 0.4em\relax ACM, 2014, pp. 1003--1014.

\bibitem{shaffer1995multiple}
J.~P. Shaffer, ``Multiple hypothesis testing,'' \emph{Annual review of
  psychology}, vol.~46, no.~1, pp. 561--584, 1995.

\bibitem{jensen2000multiple}
D.~D. Jensen and P.~R. Cohen, ``Multiple comparisons in induction algorithms,''
  \emph{Machine Learning}, vol.~38, no.~3, pp. 309--338, 2000.

\bibitem{terada2013statistical}
A.~Terada, M.~Okada-Hatakeyama, K.~Tsuda, and J.~Sese, ``Statistical
  significance of combinatorial regulations,'' \emph{Proceedings of the
  National Academy of Sciences}, vol. 110, no.~32, pp. 12\,996--13\,001, 2013.

\bibitem{tarone1990modified}
R.~Tarone, ``A modified bonferroni method for discrete data,''
  \emph{Biometrics}, pp. 515--522, 1990.

\bibitem{terada2013fast}
A.~Terada, K.~Tsuda, and J.~Sese, ``Fast westfall-young permutation procedure
  for combinatorial regulation discovery,'' in \emph{Bioinformatics and
  Biomedicine (BIBM), 2013 IEEE International Conference on}.\hskip 1em plus
  0.5em minus 0.4em\relax IEEE, 2013, pp. 153--158.

\bibitem{minato2014fast}
S.-i. Minato, T.~Uno, K.~Tsuda, A.~Terada, and J.~Sese, ``A fast method of
  statistical assessment for combinatorial hypotheses based on frequent itemset
  enumeration,'' in \emph{Machine Learning and Knowledge Discovery in
  Databases}.\hskip 1em plus 0.5em minus 0.4em\relax Springer, 2014, pp.
  422--436.

\bibitem{sugiyama2015significant}
M.~Sugiyama, F.~L. L{\'o}pez, N.~Kasenburg, and K.~M. Borgwardt, ``Significant
  subgraph mining with multiple testing correction,'' in \emph{Proceedings of
  the 2015 SIAM International Conference on Data Mining}, 2015, pp. 37--45.

\bibitem{lopez2015fast}
F.~L. L{\'o}pez, M.~Sugiyama, L.~Papaxanthos, and K.~M. Borgwardt, ``Fast and
  memory-efficient significant pattern mining via permutation testing,''
  \emph{Proceedings of the 21st ACM SIGKDD Conference on Knowledge Discovery
  and Data Mining}, 2015.

\bibitem{benjamini2010simultaneous}
Y.~Benjamini, ``Simultaneous and selective inference: current successes and
  future challenges,'' \emph{Biometrical Journal}, vol.~52, no.~6, pp.
  708--721, 2010.

\bibitem{deshpande2005frequent}
M.~Deshpande, M.~Kuramochi, N.~Wale, and G.~Karypis, ``Frequent
  substructure-based approaches for classifying chemical compounds,''
  \emph{IEEE Transactions on Knowledge and Data Engineering}, vol.~17, no.~8,
  pp. 1036--1050, 2005.

\bibitem{saigo2008partial}
H.~Saigo, N.~Kr{\"a}mer, and K.~Tsuda, ``Partial least squares regression for
  graph mining,'' in \emph{Proceedings of the 14th ACM SIGKDD international
  conference on Knowledge discovery and data mining}.\hskip 1em plus 0.5em
  minus 0.4em\relax ACM, 2008, pp. 578--586.

\bibitem{saigo2009gboost}
H.~Saigo, S.~Nowozin, T.~Kadowaki, T.~Kudo, and K.~Tsuda, ``gboost: a
  mathematical programming approach to graph classification and regression,''
  \emph{Machine Learning}, vol.~75, no.~1, pp. 69--89, 2009.

\bibitem{ketkar2009gregress}
N.~S. Ketkar, L.~B. Holder, and D.~J. Cook, ``gregress: Extracting features
  from graph transactions for regression,'' in \emph{Twenty-First International
  Joint Conference on Artificial Intelligence}, 2009.

\bibitem{fithian2014optimal}
W.~Fithian, D.~Sun, and J.~Taylor, ``Optimal inference after model selection,''
  \emph{arXiv preprint arXiv:1410.2597}, 2014.

\bibitem{kudo2004application}
T.~Kudo, E.~Maeda, and Y.~Matsumoto, ``An application of boosting to graph
  classification,'' in \emph{Advances in neural information processing
  systems}, 2004, pp. 729--736.

\bibitem{nakagawaKDD2016submitted}
K.~Nakagawa, S.~Suzumura, M.~Karasuyama, K.~Tsuda, and I.~Takeuchi, ``Safe
  pattern pruning: an efficient approach for predictive pattern mining,'' 2016,
  unpublished manuscript.

\bibitem{friedman2001elements}
J.~Friedman, T.~Hastie, and R.~Tibshirani, \emph{The elements of statistical
  learning}.\hskip 1em plus 0.5em minus 0.4em\relax Springer series in
  statistics Springer, Berlin, 2001, vol.~1.

\bibitem{lee2014exact}
J.~D. Lee and J.~E. Taylor, ``Exact post model selection inference for marginal
  screening,'' in \emph{Advances in Neural Information Processing Systems},
  2014, pp. 136--144.

\bibitem{takigawa2013graph}
I.~Takigawa and H.~Mamitsuka, ``Graph mining: procedure, application to drug
  discovery and recent advances,'' \emph{Drug discovery today}, vol.~18, no.~1,
  pp. 50--57, 2013.

\bibitem{jiawei2006datamining}
J.~Han, M.~Kamber, and J.~Pei, \emph{Data Mining: Concepts and Techniques},
  2nd~ed.\hskip 1em plus 0.5em minus 0.4em\relax Morgan Kaufmann, 2006.

\bibitem{lancichinetti2011finding}
A.~Lancichinetti, F.~Radicchi, J.~J. Ramasco, S.~Fortunato \emph{et~al.},
  ``Finding statistically significant communities in networks,'' \emph{PloS
  one}, vol.~6, no.~4, p. e18961, 2011.

\bibitem{weill2009development}
N.~Weill and D.~Rognan, ``Development and validation of a novel protein- ligand
  fingerprint to mine chemogenomic space: application to g protein-coupled
  receptors and their ligands,'' \emph{Journal of chemical information and
  modeling}, vol.~49, no.~4, pp. 1049--1062, 2009.

\bibitem{borgelt2002mining}
C.~Borgelt and M.~R. Berthold, ``Mining molecular fragments: Finding relevant
  substructures of molecules,'' in \emph{Data Mining, 2002. ICDM 2003.
  Proceedings. 2002 IEEE International Conference on}.\hskip 1em plus 0.5em
  minus 0.4em\relax IEEE, 2002, pp. 51--58.

\bibitem{yan2002gspan}
X.~Yan and J.~Han, ``gspan: Graph-based substructure pattern mining,'' in
  \emph{Data Mining, 2002. ICDM 2003. Proceedings. 2002 IEEE International
  Conference on}.\hskip 1em plus 0.5em minus 0.4em\relax IEEE, 2002, pp.
  721--724.

\bibitem{rhee2003human}
S.-Y. Rhee, M.~J. Gonzales, R.~Kantor, B.~J. Betts, J.~Ravela, and R.~W.
  Shafer, ``Human immunodeficiency virus reverse transcriptase and protease
  sequence database,'' \emph{Nucleic acids research}, vol.~31, no.~1, pp.
  298--303, 2003.

\bibitem{saigo2007mining}
H.~Saigo, T.~Uno, and K.~Tsuda, ``Mining complex genotypic features for
  predicting hiv-1 drug resistance,'' \emph{Bioinformatics}, vol.~23, no.~18,
  pp. 2455--2462, 2007.

\end{thebibliography}

\appendix
\section{Proofs}

\paragraph{Proof of Proposition~\ref{prop:search-over-k-trees}}
\begin{proof}
%
%
From \eq{eq:linear-event}, 
the constraint
$\bm{y} + \theta \bm{\eta} \in {\rm Pol}(\cK, \cA, \cT)$
is written as
%
%
%
\begin{subequations} 
 \label{eq:app:prop3-a}
 \begin{align}
  \frac{(\bm \tau_{j} - \bm \tau_{j^\prime})^\top \bm y}{(\bm \tau_{j^\prime} - \bm \tau_{j})^\top \bm \eta}
  \le \theta
  ~~\text{if}~~ (\bm \tau_{j^\prime} - \bm \tau_{j})^\top \bm \eta < 0, \\
 \frac{(\bm \tau_{j} - \bm \tau_{j^\prime})^\top \bm y}{(\bm \tau_{j^\prime} - \bm \tau_{j})^\top \bm \eta}
  \ge \theta
  ~~\text{if}~~ (\bm \tau_{j^\prime} - \bm \tau_{j})^\top \bm \eta > 0\phantom{,}
 \end{align}
\end{subequations}
for all possible pairs of $(j, j^\prime) \in \cK \times \{[J] \setminus \cK\}$.
(i) First, for $(j, j^\prime) \in \cK \times \{[J] \setminus \cK\}$
such that  
$(\bm \tau_{j^\prime} - \bm \tau_{j})^\top \bm \eta < 0$,
the minimum possible feasible $\theta$ would be
\begin{align*}
  \max_{(j, j^\prime) \in \cK \times \{[J] \setminus \cK\}}
  \frac{(\bm \tau_{j} - \bm \tau_{j^\prime})^\top \bm y}{(\bm \tau_{j^\prime} - \bm \tau_{j})^\top \bm \eta},
\end{align*} 
and 
the maximum possible feasible $\theta$ would be $\infty$. 
(ii) 
Similarly,  
for $(j, j^\prime) \in \cK \times \{[J] \setminus \cK\}$
such that  
$(\bm \tau_{j^\prime} - \bm \tau_{j})^\top \bm \eta > 0$,
the minimum possible feasible $\theta$ would be $-\infty$ 
and 
the maximum possible feasible $\theta$ would be
\begin{align*}
  \min_{(j, j^\prime) \in \cK \times \{[J] \setminus \cK\}}
  \frac{(\bm \tau_{j} - \bm \tau_{j^\prime})^\top \bm y}{(\bm \tau_{j^\prime} - \bm \tau_{j})^\top \bm \eta}.
\end{align*} 
Since the requirements in (i) and (ii) must be satisfied for all possible  
$(j, j^\prime) \in \cK \times \{[J] \setminus \cK\}$,
by combining (i) and (ii),
$\theta_{\rm min}$
and
$\theta_{\rm max}$
are given by
\eq{eq:theta_min}
and
\eq{eq:theta_max},
respectively. 
\end{proof}

\paragraph{Proof of Theorem~\ref{theo:pruning_cond}}
\begin{proof}
Noting that $0 \le \tau_{i, \ell^\prime} \le \tau_{i, j^\prime} \le 1$,
for any descendant node $\ell^\prime \in Des(j^\prime)$
\begin{subequations} 
 \begin{align}
  \nonumber
  (\bm \tau_{j} - \bm \tau_{\ell^\prime})^\top \bm y
  &= \bm \tau_{j}^\top \bm y \!-\! \sum_{i : y_i > 0}\! \tau_{i, \ell^\prime} y_i \!-\! \sum_{i : y_i < 0}\! \tau_{i, \ell^\prime} y_i \\
  \label{eq:proof-theo4-a1}
  & \hspace*{-5mm}
  \ge \bm \tau_{j}^\top \bm y \!-\! \sum_{i : y_i > 0}\! \tau_{i, \ell^\prime} y_i
   \ge \bm \tau_{j}^\top \bm y \!-\! \sum_{i : y_i > 0}\! \tau_{i, j^\prime} y_i,\\
  \nonumber
  (\bm \tau_{\ell^\prime} - \bm \tau_{j})^\top \bm \eta
  &= \sum_{i : \eta_i > 0}\! \tau_{i, \ell^\prime} \eta_i \!+\! \sum_{i : \eta_i < 0}\! \tau_{i, \ell^\prime} \eta_i
     \!-\! \bm \tau_{j}^\top \bm \eta \\
  \label{eq:proof-theo4-a2}
  & \hspace*{-5mm}
  \ge \sum_{i : \eta_i < 0}\! \tau_{i, \ell^\prime} \eta_i \!-\! \bm \tau_{j}^\top \bm \eta
   \ge \sum_{i : \eta_i < 0}\! \tau_{i, j^\prime} \eta_i \!-\! \bm \tau_{j}^\top \bm \eta.
\end{align}
\end{subequations}
We prove the first half of the theorem.
(i)
%
From
\eq{eq:proof-theo4-a2}, 
\begin{align*}
 \text{\eq{eq:theo-main-a1}}
 ~\Rightarrow~
 (\bm \tau_{\ell^\prime} - \bm \tau_{j})^\top \bm \eta \ge 0. 
\end{align*}
Also from Proposition~\ref{prop:search-over-k-trees}, 
any pairs $(j, \ell^\prime)$ such that $(\bm \tau_{\ell^\prime} - \bm \tau_{j})^\top \bm \eta \ge 0$
are irrelevant to 
the solution $\theta_{\rm min}$. 
It means that,
when
\eq{eq:theo-main-a1}
holds,  
$(j, \ell^\prime)$
for
$\ell^\prime \in Des(j^\prime)$
do not affect the solution of \eq{eq:theta_min}.  
(ii)
%
From Proposition~\ref{prop:search-over-k-trees},
we only need to consider the case where
$(\bm \tau_{j^\prime} - \bm \tau_j)^\top \bm \eta < 0$
and  
$(\bm \tau_{\ell^\prime} - \bm \tau_j)^\top \bm \eta < 0$. 
When
$\bm \tau_{j}^\top \bm y \!-\! \sum_{i : y_i > 0}\! \tau_{i, j^\prime} y_i \ge 0$,
from \eq{eq:proof-theo4-a1}, 
 \begin{align*}
 \frac{
  (\bm \tau_{j} - \bm \tau_{\ell^\prime})^\top \bm y
 }{
  \min[(\bm \tau_{\ell^\prime} - \bm \tau_{j})^\top \bm \eta, 0]
 }
 &\le
 \frac{
  \bm \tau_{j}^\top \bm y \!-\! \sum_{i : y_i > 0}\! \tau_{i, \ell^\prime} y_i
 }{
  \min[\sum_{i : \eta_i < 0}\! \tau_{i, \ell^\prime} \eta_i \!-\! \bm \tau_{j}^\top \bm \eta, 0]
 } \\
 &\le
 \frac{
  \bm \tau_{j}^\top \bm y \!-\! \sum_{i : y_i > 0}\! \tau_{i, j^\prime} y_i
 }{
  \min[\sum_{i : \eta_i < 0}\! \tau_{i, j^\prime} \eta_i \!-\! \bm \tau_{j}^\top \bm \eta, 0]
 }
 \le
\hat{\theta}_{\rm min}^{\cV}.
 \end{align*}
%
It means that, 
when
\eq{eq:theo-main-a2}
holds,  
$(j, \ell^\prime)$
for
$\ell^\prime \in Des(j^\prime)$
do not affect the solution of \eq{eq:theta_min}.  
By combining (i) and (ii),
the first half of the theorem is proved.
The latter half of the theorem can be shown similarly.
\end{proof}

\end{document}